\documentclass{article}

\usepackage{amsmath,amssymb,pifont}
\usepackage{multicol}
\usepackage{amstext}
\usepackage{amsthm}
\usepackage{multirow}
\usepackage{booktabs}
\usepackage[skip=0pt]{subcaption}
\usepackage{times}
\usepackage{lipsum}
\usepackage[shortlabels]{enumitem}
\usepackage{cancel}
\usepackage{wrapfig}
\usepackage{array}
\usepackage{siunitx}
\usepackage{csvsimple}
\usepackage[multidot]{grffile}
\usepackage{bbm}
\usepackage{dblfloatfix}
\usepackage{hyperref}
\usepackage{makecell}
\usepackage{bbm, dsfont}
\usepackage{mathtools}
\usepackage{comment}

\usepackage{geometry}
\usepackage{mathpazo}
\usepackage{enumitem}

\usepackage[multiple]{footmisc}
\usepackage{mathrsfs}
\usepackage[disable]{todonotes}
\usepackage{tikz}
\usepackage{hyperref}
\usepackage{cleveref}

\usepackage{algpseudocode,algorithm,algorithmicx}
\usepackage{xfrac}

\usepackage{graphicx} %
\usepackage{color}

\usepackage{array}
\usepackage{amssymb}
\usepackage{amsmath}
\usepackage{xspace}
\usepackage{fancyhdr}
\usepackage{comment}

\newcommand{\anoteinline}[1]{\todo[color=cyan!30,inline]{AT: #1}}

\newcommand{\agnote}[1]{\todo[color=purple!30]{\tiny AG: #1}}

\newcommand{\jnote}[1]{\todo[color=red!30]{\tiny JU: #1}}

\hypersetup{final}

\newcommand{\boldzero}{\ensuremath{\boldsymbol{0}}}

\newcommand{\bfc}{\ensuremath{\mathbf{c}}}

\newcommand{\bfv}{\ensuremath{\mathbf{v}}}

\newcommand{\bfx}{\ensuremath{\mathbf{x}}}

\newcommand{\calA}{\ensuremath{\mathcal{A}}}

\newcommand{\calC}{\ensuremath{\mathcal{C}}}
\newcommand{\calD}{\ensuremath{\mathcal{D}}}

\newcommand{\calG}{\ensuremath{\mathcal{G}}}

\newcommand{\calL}{\ensuremath{\mathcal{L}}}
\newcommand{\calM}{\ensuremath{\mathcal{M}}}
\newcommand{\calN}{\ensuremath{\mathcal{N}}}

\newcommand{\calP}{\ensuremath{\mathcal{P}}}
\newcommand{\calQ}{\ensuremath{\mathcal{Q}}}
\newcommand{\calR}{\ensuremath{\mathcal{R}}}

\renewcommand{\Pr}{\mathop{\mathbf{Pr}}}

\newcommand{\E}{\mathop{\mathbf{E}}}

\newtheorem{lem}{Lemma}[section]

\newtheorem{thm}[lem]{Theorem}
\newtheorem{infthm}[lem]{Informal Theorem}

\newtheorem{defn}[lem]{Definition}
\newtheorem{fact}[lem]{Fact}

\DeclareMathOperator*{\argmax}{arg\,max}
\DeclareMathOperator*{\argmin}{arg\,min}

\makeatletter
\newcommand{\vast}{\bBigg@{4}}
\newcommand{\Vast}{\bBigg@{5}}
\makeatother

\newcommand{\ex}[2]{{\ifx&#1& \mathbb{E} \else
\underset{#1}{\mathbb{E}} \fi \left[#2\right]}}
\newcommand{\pr}[2]{{\ifx&#1& \mathbb{P} \else
\underset{#1}{\mathbb{P}} \fi \left[#2\right]}}

\newcommand{\privT}{\theta^{\sf priv}\,}

\newcommand{\ltwo}[1]{\left\|#1\right\|_2}

\newcommand{\grad}{\nabla}

\usepackage{ulem}

\DeclarePairedDelimiterX{\infdivx}[2]{(}{)}{%
  #1\;\delimsize\|\;#2%
}

\newcommand{\mypar}[1]{\smallskip
	\noindent{\textbf{{#1}:}}}
	
\renewcommand{\epsilon}{\varepsilon}

\renewcommand{\tilde}{\widetilde}

\newcommand{\dist}{\calD}

\newcommand{\supp}{\text{supp}}

\newcommand{\set}[1]{\left\{ {#1} \right\}}
\newcommand{\norm}[1]{{\left\Vert {#1} \right\Vert}}
\newcommand{\paren}[1]{\left( {#1} \right)}
\newcommand{\sparen}[1]{\left[ {#1} \right]}

\setlist{nolistsep}
\setlist[itemize]{noitemsep, topsep=0pt}

\newcommand{\dataset}{D}

\newcommand{\domain}{\tau}

\newcommand{\lip}{L}

\setlist{nolistsep}
\setlist[itemize]{noitemsep, topsep=0pt}

\newtheorem{theorem}[lem]{Theorem}

\newcommand{\thetapriv}{\theta^{\mathsf{priv}}}

\newcommand{\Exp}[1]{\mathbb{E}\left[#1\right]}
\newcommand{\Expect}[2]{\mathbb{E}_{#1}\left[#2\right]}

\newcommand{\riskerm}{{\sf Risk}_{\sf ERM}}
\newcommand{\riskpop}{{\sf Risk}_{\sf SCO}}
\newcommand{\threshold}{\psi}
\newcommand{\rashomonemp}{\calG}

\newcommand{\slack}{\gamma}
\newcommand{\tvd}{\lambda}
\newcommand{\gap}{\kappa}

\usepackage{fullpage}
\usepackage[numbers, sort, comma, square]{natbib}

\newtheorem{remark}{Remark}
\newtheorem{definition}[lem]{Definition}

\begin{document}

\normalem

\title{Differentially Private Sampling from Rashomon Sets, and the Universality of Langevin Diffusion for Convex Optimization}

\author{Arun Ganesh\thanks{Google Research.  \texttt{arunganesh@google.com}. Part of this work was done at UC Berkeley while being supported in part by NSF CCF-1816861.} 
\and 
Abhradeep Thakurta\thanks{Google DeepMind. \texttt{athakurta@google.com}.} 
\and 
Jalaj Upadhyay\thanks{Rutgers University. \texttt{jalaj.upadhyay@rutgers.edu}. This work was supported by the Decanal Research grant from Rutgers University.}
}
\maketitle

\begin{abstract}
In this paper we provide an algorithmic framework based on Langevin diffusion (LD) and its corresponding discretizations that allow us to simultaneously obtain: i) An algorithm for sampling from the exponential mechanism~\citep{mcsherry2007mechanism}, whose privacy analysis does not depend on convexity and which can be stopped at anytime without compromising privacy, and ii) tight uniform stability guarantees for the exponential mechanism. As a direct consequence, we obtain optimal excess empirical and population risk guarantees for (strongly) convex losses under both pure and approximate differential privacy (DP). The framework allows us to design a DP uniform sampler from the  Rashomon set. Rashomon sets are widely used in interpretable and robust machine learning, understanding variable importance, and characterizing fairness.

{\color{blue} Note: For ease of presentation, some results appear in the previous version of this paper on arXiv (v3) that do not appear in this version, nor are subsumed by results in this version. Please see Section~\ref{sec:omitted} for more details.}
\end{abstract}

\section{Introduction}
\label{sec:intro}

\anoteinline{Anywhere we state $\calL(\theta^*;D)$, we need to define $\theta^*=\argmin\limits_{\theta\in\calC}\calL(\theta;D)$ unless the context is clear.}

\textit{Differentially private  empirical risk minimization} (DP-ERM)~\citep{chaudhuri2011differentially, ChourasiaYS21, iyengar2019towards, kifer2012private,BST14,smith2017interaction, song2013stochastic, song2020characterizing, WLKCJN17} and \emph{differentially  private stochastic optimization} (DP-SCO)~\citep{asi2021adapting, bassily2019private,bassily2020stability,feldman2019private,kulkarni2021private,gopi2022private} are two of the most widely studied problems in the differential privacy (DP) literature. The optimal algorithms for either of these settings are either based on {\em differentially private stochastic gradient descent} (DP-SGD)~\citep{DP-DL,BST14,song2013stochastic}), or sampling from an appropriate Gibbs distribution (a.k.a.~the exponential mechanism~\citep{mcsherry2007mechanism}). In this paper we revisit the sampling perspective of DP optimization and study its implications. 

At a high level, the Gibbs distribution sampling problem is to generate a sample $\theta$ from the distribution with density proportional to $\exp(-\beta \mathcal L(\theta;D))$. Here $\beta$ is known as the {\em inverse temperature}, and $\calL(\theta;D)=\frac{1}{n}\sum\limits_{i=1}^n\ell(\theta;d_i)$ (with $\ell:\mathbb{R}^p\times\tau\to\mathbb{R}$) is the {\em empirical loss function}. Our main contribution is an algorithmic framework  based on Langevin Diffusion (LD) (see \Cref{fig:LD}) to privately (approximately) sample from a Gibbs distribution, with the following implications: 

\begin{enumerate}
    \item Our framework recovers all existing (and tight) bounds for DP-ERM~\citep{BST14}, and in some cases improves on the best known bounds. The framework provides tight $O(1/n)$-uniform stability guarantees for the Gibbs distribution  on strongly convex losses and $O(1/\sqrt{n})$-uniform stability on mildly regularized convex losses. This is much tighter than the generic $O(\epsilon)$-uniform stability provided by the DP guarantee~\citep{BST14}\footnote{The uniform stability arguments in \cite{BST14} gives the $O(\epsilon + \frac{p}{\epsilon n})$ SCO bound for convex losses under pure-DP. This is at best $O(\sqrt{p/n})$. In contrast, we obtain the optimal $O(1/\sqrt{n} + p / \epsilon n)$.}.     This allows us to obtain optimal DP-SCO bounds under both $\epsilon$-DP and $(\epsilon,\delta)$-DP\footnote{Through out the paper we will assume $\epsilon=O(\log(1/\delta))$.}, and improves on~\cite{asi2021adapting} for the  $\epsilon$-DP case.  In this sense, the LD framework is universal for DP optimization. 

    \item The privacy guarantee of our Langevin Diffusion (LD) based algorithm does not depend on convexity of the loss function,  $\calL(\theta;D)$. Therefore, it can be used in non-convex settings without compromising privacy. 

    \item In the $(\epsilon,\delta)$-DP setting, we can release the complete trajectory of the LD until it reaches the stationary Gibbs distribution. As a direct consequence, our algorithm is ``anytime'' DP, i.e., the privacy is not compromised even if we stop before the chain converges to the stationary distribution. This is a very useful property in practice; in fact, subsequent works~\citep{shejwalkar2022recycling,rabanser2023training} have shown that the path of LD can be used to quantify predictive uncertainty.
\end{enumerate}

\paragraph{Sampling from Rashomon sets} Our framework also allows us to go beyond what can be achieved by known algorithms for private learning. In particular, it allows us to  uniformly and privately sample from the {\em Rashomon set}~\citep{breiman2001statistical} (Definition~\ref{def:privatesampler}), which has been extensively studied in interpretable and robust machine learning~\citep{rudin2022interpretable}, understanding spectrum of variable importance~\citep{fisher2019all}, decision making~\citep{tulabandhula2013machine, tulabandhula2014combining, tulabandhula2014robust}, measuring underspecification~\citep{madras2019detecting}, and characterizing fairness~\citep{coston2021characterizing}. At a high level, Rashomon set is the set of  equally-performing models in terms of training/testing loss. 

\subsection{Related Work}
We start by giving a brief exposition of some of the related work and a literature survey of works in Rashomon set, differentially private learning, and dynamical systems.

\paragraph{Rashomon set and predictive multiplicity} 
Rashomon set has been extensively studied in applied machine learning since its conception~\citep{coker2021theory, fisher2019all,  meinshausen2010stability, letham2016prediction, nevo2017identifying,semenova2022existence, srebro2010smoothness, tulabandhula2014robust} (see the survey~\citep{rudin2022interpretable} for more references), culminating in a recent  work of \cite{semenova2022existence}. For example, \cite{fisher2019all} leverages the Rashomon set in order to understand the spectrum of variable importance and other statistics across the set of good models, and \cite{tulabandhula2013machine, tulabandhula2014combining, tulabandhula2014robust} uses the Rashomon set to assist with decision making. However, it is computationally inefficient to find the simplest model in the Rashomon set, so a natural question is when should we even search for a simpler model. In a recent work, ~\cite{semenova2022existence} showed that, if there is a large Rashomon set of almost-equally-accurate models, a simple model may also be contained in it and that model is guaranteed to generalize well. For this, they defined {\em Rashomon ratio}, which is  the ratio of the volume of the set of accurate models to the volume of the hypothesis space. They then used the insights gained from Rashomon ratios to infer whether a simpler model exists or not.  

Rashomon ratios are defined in terms of volumes of the hypothesis (or model parameter) space. In a recent work, ~\cite{hsurashomon} proposed a different metric known as {\em Rashomon capacity}. It aims to measure the multiplicity of classifier outputs for individual samples, i.e., it measures the spread of the scores with divergence measures for probability distributions (also known as \emph{predictive multiplicity}). In particular, this helps to distinguish Rashomon sets where different predictions are a result of highly different predicted probabilities vs sets where the predictions are different but the come from similar soft outputs.

\paragraph{Some applications of Rashomon sets} Rashomon set has many applications, such as in interpretable and robust machine learning~\citep{fisher2019all, rudin2022interpretable}, understanding spectrum of variable importance~\citep{fisher2019all}, decision making~\citep{tulabandhula2013machine, tulabandhula2014combining, tulabandhula2014robust}, measuring underspecification~\citep{madras2019detecting}, and characterizing fairness~\citep{coston2021characterizing} to name a few. In fact, if we assume that the loss function is smoothness of a loss function, then one can  obtain
a tighter excess risk bound through local Rademacher complexity~\citep{bartlett2005local}. A line of work has also related Rashomon sets with $p$-hacking and robustness of estimation. The central argument there is that the Rashomon
set is a set on which one might conduct a sensitivity analysis for choices made by an analyst. We refer the interested readers to the survey by \cite{rudin2022interpretable}.

\paragraph{Sampling and differential privacy} Several lines of work have designed Markov chains that generate samples from distributions that are close to a given log-concave distribution. For example, there are works gives sampling algorithms with bounds on the distance to the target density  in terms of Wasserstein distance~\citep{dalalyan2020sampling, durmus2017nonasymptotic}, KL-divergence~\citep{durmus2019analysis}, and Renyi divergence~\citep{vempala2019renyi}. For privacy, we need a bound in terms of $\ell_\infty$ distance. For this, the first work that perform sampling with bounded $\ell_\infty$ distance is by \cite{hardt2010geometry}. This was extended to bounded Lipschitz log-concave distribution by \cite{BST14}. Since then, several works have shown efficient algorithms for sampling from log-concave distribution~\citep{mangoubi2021sampling, GaneshT20}. There has been some work that find polynomial time sampling algorithms for special loss functions. For example, \cite{asi2020instance, asi2020near} showed efficient algorithms to sample from the exponential mechanism when the score function has a specific structure, which they call {\em path-length function}. The motivation of \cite{asi2020instance, asi2020near} was to study instance-optimality of certain wide class of statistical problem. A variant of this function (which is quasi-convex) has been recently used by \cite{hopkins2022robustness} for robust high-dimensional parameter estimation
problems, including mean and covariance estimation, when the data is picked from multivariate Gaussian distribution. Using the insights from differential privacy, a recent line of work~\citep{altschuler2022privacy, altschuler2022resolving} have improved (and characterized) the mixing time of the discretization of Langevin diffusion. 

There has been some recent work to study the asymptotic bias introduced by the discretization of Langevin diffusion. In a recent work, \cite{altschuler2022concentration} showed that the stationary distribution of discretization of the Langevin diffusion is sub-Gaussian when the potential function is strongly convex, and  is sub-exponential when the potential function is convex.

\paragraph{Differential privacy and dynamical systems}
The connection between dynamical systems and differential privacy is also not new.~\cite{ChourasiaYS21} and~\cite{ryffel2022differential} study discretization of the LD algorithm as DP-(Stochastic) Gradient Langevin Dynamics (DP-SGLD). They show that under smoothness and strong convexity on the loss function $\calL(\theta;D)$, the privacy cost of DP-SGLD converges to a stationary finite value, even when the number of time steps goes to $\infty$.~\cite{wang2019differentially} used the result by~\cite{raginsky17a} to prove a sub-optimal excess empirical risk of $\widetilde O \left( \frac{p\log(1/\delta)}{\epsilon^2\log(n)} \right)$ for non-convex loss functions. In a concurrent, and complementary work on convex losses,~\cite{gopi2022private} study private optimization and show the universality of exponential mechanisms for both stochastic convex optimization and empirical risk minimization. Their analysis takes the sampling perspective when the diffusion process has  completed.

It is probably important to mention that objective perturbation~\citep{chaudhuri2011differentially,kifer2012private} can be potentially thought of as a (near) universal algorithm for the problem classes considered in this paper, albeit the following two caveats: i) The instantiation of the algorithm for $\epsilon$-DP and $(\epsilon,\delta)$-DP require two different noise models to be drawn from, namely, Gamma distribution, and Normal distribution, and ii) It requires the loss functions $\ell(\theta;\cdot)$ to be twice-continuously differentiable, and $\nabla^2_\theta\ell(\theta;\cdot)$ to have a near constant rank.
As mentioned in the remainder of our paper, Langevin diffusion does not require any such assumptions.\footnote{In particular, we can always ensure twice differentiability by convolving the loss function with the bump kernel~\citep{kifer2012private}, and then make the smoothness parameter finite but arbitrarily large which does not affect the Lipschitzness.}

Recently, \cite{mangoubi2022re} used continuous-time viewpoint to study the error incurred by adding a symmetric Gaussian matrix to input covariance matrix. In particular, they viewed the the perturbed matrix as a continuous-time symmetric matrix diffusion, where each entry of the perturbed matrix  is the value reached by a 
Brownian motion after the time equals to the scaling of variance required for privacy. In particular, the corresponding Brownian motion is well studied in statistical quantum physics and is known as {\em Dyson brownian motion}. 

There is a contemporary and most closely related work of~\cite{gopi2022private} to ours. We defer the comparison to \Cref{sec:gopietal}\footnote{The claim of contemporarity is also supported by the authors of~\cite{gopi2022private}.}.

\subsection{Our Contributions}

Our main contribution is to design a Langevin diffusion (LD) based DP sampler for the following Gibbs distribution: 
$$\exp(-\beta \max\{\calL(\theta; \dataset), \threshold + \min\limits_{\theta^*\in \calC} \calL(\theta^*;D)\}).$$ 

As we will see in Section~\ref{sec:applications}, by setting $\threshold=0$ we obtain optimal DP-ERM/DP-SCO algorithms, and for $\threshold>0$ we obtain a DP Rashomon set sampler. 
In this section, {we first state the main result followed by the uniform stability result for LD. We end with a discussion of discretization of our LD, that outputs a sample within $\delta$ total variation distance of the stationary distribution of LD at the same privacy/utility trade-off.}

We  start with the LD algorithm, described in~\Cref{fig:LD}, which forms the building block for all the algorithms considered in this paper.  {Intuitively, one should think of \eqref{eq:LD} as the limit of noisy gradient descent and \eqref{eq:PLD} as the limit of projected noisy gradient descent, both as step size $\eta \rightarrow 0$.} Here and throughout this paper, $O_\delta(\cdot)$ and $\widetilde{O}_\delta(\cdot)$ hides polylog factors in $1/\delta$.

\begin{figure}[t]
\fbox{
\begin{minipage}{0.95\textwidth}
{\bf Langevin diffusion (LD).} Let $W_t$ be a $p$-dimensional Brownian motion and $\beta>0$ be the {\em inverse temperature}. Then LD is the following stochastic differential equation:
\begin{equation}
    d\theta_t=-\beta \nabla\calL(\theta_t;\dataset)\cdot dt+\sqrt {2}\cdot dW_t.
    \label{eq:LD}
\end{equation}

{\bf ``Projected" Langevin diffusion.} Sometimes, we only have the Lipschitz guarantee within a constrained set. We can also consider the following ``projected'' version of LD: 

{
\begin{equation}
    d\theta_t=-\beta \nabla\calL(\theta_t;\dataset)\cdot dt+\sqrt{2}\cdot dW_t - \nu_t \mu(dt)
    , \forall t \geq 0: \theta_t \in \calC .\label{eq:PLD}
\end{equation}
where $\mu$ is a measure supported on $\{t: \theta_t \in \partial \calC\}$ and $\nu_t$ is an outer unit normal vector at $\theta_t$ for all such $\theta_t$.} See \cite[Section 2.1, 3.1]{BubeckPLD} for a discussion of \eqref{eq:PLD}.
\end{minipage}
}  
\caption{(Projected) Langevin diffusion}
\label{fig:LD}

\end{figure}

\begin{infthm}[Corresponds to Theorems~\ref{thm:expMech-rashomon} and~\ref{thm:expMech-rashomon-approx}]
Assume that the loss functions are $1$-Lipschitz, and the constraint set $\calC$ has diameter at most one. Then there exists a LD process $\{\theta_t\}_{t \geq 0}$ with stationary distribution $\Theta_{\infty}$
proportional to $\exp(-\beta \max\{\calL(\theta; \dataset), \threshold + \min\limits_{\theta^*\in \calC} \calL(\theta^*;D)\})$, and
 \begin{enumerate}
     \item [(i)] $\Theta_\infty$ satisfies $\epsilon$-DP if $\beta = O(\epsilon n)$.
     \item [(ii)] If the loss function is $m$-strongly convex and $M$-smooth, then for $t = \widetilde{O}_\delta\left(\frac{1}{\beta m}\right)$ and $\beta =\widetilde{O}_\delta\left(m \epsilon^2 \min\{n^2, {1 \over M \threshold}\} \right)$, releasing $\{\theta_{t'}\}_{0 \leq t' \leq t}$ is $(\epsilon, \delta)$-DP and $\theta_T$ is within total variation distance (TVD) $\delta$ of $\Theta_\infty$.
 \end{enumerate}
 Furthermore, the privacy guarantee only requires Lipschitzness and smoothness. 
\label{infthm:LD}
\end{infthm}

A key takeaway from the privacy guarantee of Theorem~\ref{infthm:LD} is that, as $\threshold$ becomes smaller, one can run the LD at a higher $\beta$ and thus provide stronger risk guarantees. In particular, we show an excess empirical risk bound of $p/\beta$) in Theorem~\ref{thm:expConv-rashomon-util}. In fact, if $\threshold \leq 1/n^2$, then the choice of $\beta$ is an approximately $n$ times more than that in the $\epsilon$-DP case. We believe the relation $\beta \propto \min\{1 / \threshold, n^2\}$ is necessary. (See Section~\ref{sec:RPrivacy} for a formal reasoning.) 

For part (i) in Theorem~\ref{infthm:LD}, the privacy follows from the analysis of the exponential mechanism. 
For part (ii), we use a continuous analog of the composition theorem for R\'enyi-DP (see Lemma~\ref{lem:renyifinite} and more discussion on our continuous time composition theorem below). To show the sampling guarantee, we show that the stationary distribution satisfies a {\em log Sobolev inequality} (a measure of concentration; see Definition~\ref{def:LSI}) using standard techniques known as the {\em Bakry-Emery criterion} and  {\em Holley-Stroock perturbation principle} (see e.g., Appendix A of \cite{Schlichting_2019}). The convergence guarantee then follows using the results in  \cite{vempala2019renyi}. 

Note that part (ii) of Theorem~\ref{infthm:LD} also shows that $\Theta_\infty$ is private via analyzing privacy of the chain rather than the stationary distribution. This, in particular, shows that the entire trajectory of the LD is private (not just the~\emph{final iterate}). This matters in practice as works such as  \citep{shejwalkar2022recycling,rabanser2023training} have shown improved performance and uncertainty estimation from using intermediate values. 

There is another advantage of analyzing the privacy of the entire chain. Unlike the sampling/utility guarantee, the privacy in part (ii) does not rely on convexity, i.e., we can use it for non-convex loss functions. Furthermore, by taking $\threshold = 0$ and comparing to the Gaussian mechanism~\citep{ODO}, we can see our privacy guarantee is tight up to log factors.

\paragraph{Uniform stability of Langevin diffusion (Section~\ref{sec:roshPOP})}
While empirical guarantees are useful in their own regards, it is often desirable to get population risk guarantee. We derive a population risk guarantee by showing the uniform stability property of the LD on thresholded losses. This implies that any empirical accuracy guarantee for the Gibbs sampler in Theorem~\ref{infthm:LD} also extends to population risk guarantees:

\begin{infthm}[Corresponds to Theorem~\ref{thm:rashomon-uniform}]
Under the same assumptions and choice of $\beta$ as in Theorem~\ref{infthm:LD}, the LD at any time (including at its stationary distribution) satisfies $O(L \sqrt{{\threshold}/{m}} + {L^2}/\paren{m n})$-uniform stability. 
\label{infthm:PopR}
\end{infthm}
 The proof uses the fact that the time-independent uniform stability for finite-time LD implies the same uniform stability for its Gibbs distribution, which could be of independent interest. This in particular gives optimal SCO rates under $\epsilon$-DP guarantee that matches non-private bound of $O(1/\sqrt{n})$ as $\epsilon \to \infty$, thereby, improving on the state-of-the-art results  (\Cref{sec:applications}). 

\paragraph{Continuous time composition for LD (Section~\ref{sec:RPrivacy})} 
We cannot use standard composition theorems of DP~\citep{dwork2014algorithmic} because the underlying algorithm is a continuous time process. 
We quantify the R\'enyi divergence between two LD processes when run on neighboring data sets. 
A similar result was also provided in \cite[Theorem 1]{ChourasiaYS21} only for the last iterate, $\theta_t$. 
In contrast, we prove a divergence bound between the entire histories $\{\theta_{t'}\}_{0 \leq t' \leq t}$, which enables us to output weighted averages of $\theta_{t'}$'s privately. Furthermore, it is proven using only tools from the differential privacy literature and \emph{Fatou's lemma}, providing an arguably much simpler proof. 

\agnote{I don't think our SGLD result differs too much from the previous ones in terms of complexity}

\paragraph{Discretizing our LD (Section~\ref{sec:sgld})}
In general, sampling from a continuous-time object such as LD is intractable in practice. A common technique for approximately sampling from the distribution induced by an LD is the {\em Stochastic Gradient Langevin Dynamics}  (SGLD), which has been extensively studied in the literature (e.g. \cite{dalalyanLangevin, raginsky17a, pmlr-v83-cheng18a, ChengNCLD, vempala2019renyi, GaneshT20, erdogdu2021convergence, ChewiLSILD,welling2011bayesian,ryffel2022differential,ChourasiaYS21}). SGLD uses $T$ steps of noisy gradient descent with step size $\eta$, which approximates running \cref{eq:LD} for time $t = T \eta$. Using results in \cite{ChewiLSILD}, we show that SGLD provides a private approximation (w.r.t. TVD) of our Gibbs sampler in a polynomial number of gradient oracle calls. One disadvantage of this result is that the oracle complexity has a worse dependence on the problem parameters than DP-SGD with standard hyperparameters. For example, DP-SGD's iteration complexity in \cite{bassily2019private} is constant w.r.t. dimensions as it goes to infinity, whereas our SGLD iteration complexity has a linear dependence on dimensions. We leave closing this gap as a question for future investigation.

\anoteinline{Compare the gradient complexity w.r.t. full-batch DP-SGD for DP-ERM when $\threshold=0$.}

\subsection{Applications of Our Algorithmic Framework}
\label{sec:applications}
\mypar{Recovering DP-ERM/DP-SCO bounds (Section~\ref{sec:DPERMSCO})} We show that setting $\threshold=0$ for the Gibbs sampler in Theorem~\ref{infthm:LD} retrieves the optimal DP-ERM/DP-SCO bounds, i.e., using only the LD sampler as a primitive, one can achieve all the existing bounds for DP-ERM/DP-SCO and improve some prior results (see Table~\ref{tb:resultsERMSCO}) as  corollaries. 
Since most of these results are known in the literature, in the next theorem, we only present the improvement exhibited in this paper.

\begin{infthm}[Corresponds to  Theorems~\ref{thm:expConvSCO} and~\ref{thm:expStrongConv}]
For $L$-Lipschitz convex losses over $\calC$, there exists an $\epsilon$-DP algorithm 
with  $O\left(\frac{L\ltwo{\calC}p}{\epsilon n} + \frac{L \ltwo{\calC}}{\sqrt{n}}\right)$ excess population risk. Further, if the loss function is also $m$-strongly convex, then there is an $\epsilon$-DP algorithm that has excess population risk of $O\left(\frac{L^2 p^2 \log n}{m \epsilon^2 n^2} + \frac{L^2}{m n}\right)$. 
\label{infthm:PopEpsilonConvex}
\end{infthm}

The best prior bounds were $O\left(\frac{p\log n}{\epsilon n}+\frac{\log^{3/2} n}{\sqrt n}\right)$ for convex losses, and $O\left(\frac{p^2 \log^2 n}{m\epsilon^2 n^2}+\frac{\log^3 n}{m n}\right)$ for $m$-strongly convex losses, both due to~\cite{asi2021adapting}, which lacked the ability to match the optimal non-private bounds of $O\left({1}/{\sqrt n}\right)$ (and $O\left({1}/{m n}\right)$ respectively) as $\epsilon\to\infty$. (To the best of our knowledge, this gap is inherent in the technique of~\cite{asi2021adapting}.) 

In most cases, the best DP-ERM/SCO bounds can be achieved by setting $\threshold = 0$ in Theorems~\ref{infthm:LD} and~\ref{infthm:PopR} (in the convex case, after adding a quadratic regularizer to enforce strong convexity). The second result in Theorem~\ref{infthm:PopEpsilonConvex} is the only one which does not directly apply the Rashomon sampler's risk bounds to the (regularized) loss. Instead, we use an \textit{iterated exponential mechanism}, which samples from a sequence of Gibbs distributions defined over an adaptively chosen sequence of sets $\calC_{k}\subset\calC_{k-1}\subseteq\ldots\subseteq\calC_{0}=\calC$. To analyze it, we compose the privacy/utility analysis obtained independently for the Gibbs distribution over each of these sets. Our analysis simplifies the analysis in~\cite{BST14} and does not require running two different algorithms (i.e., output perturbation and exponential mechanism) to obtain the optimal trade-off. We give the full description of the iterated exponential mechanism and its analysis in Section~\ref{sec:DPERMSCO}.

Additionally, for DP-ERM, in Appendix~\ref{sec:lb} we provide a lower bound for non-convex losses which demonstrate that, unlike for convex losses, it is not possible to achieve better utility (up to factors in $\log(1/\delta)$) than $\widetilde{O}(p/\epsilon n)$ even if we relax privacy guarantee to $(\epsilon,\delta)$-DP. We prove this result by appealing to the lower bound in~\cite{steinke2017tight}. 

\paragraph{Sampling from Rashomon Sets (Section~\ref{sec:rashomonSampling})} Our motivation to design uniform sampler for the Rashomon set  stems from quantifying predictive multiplicity for Rashomon sets (described later)~\citep{hsurashomon}\footnote{Predictive multiplicity refers to  models that achieve statistically-indistinguishable
performance on a test set assign wildly different predictions to an input sample. 
Therefore, if we naively pick a  model from a Rashomon set, it can have highly disparate  impact on the predictions on an individual test sample resulting in unfair and potentially individual level harm~\citep{smith2020algorithmic,creel2022algorithmic}.} and that it is a strict generalization of the problems of DP-ERM and DP-SCO. Given the wide applicability of Rashomon set mentioned earlier, we believe studying this problem will have more implications. 

We start with the formal definition of uniform sampling from Rashomon set.

\begin{defn}
[$(\lambda, \psi, \gamma)$-Rashomon sampler]
\label{def:privatesampler}
Given a loss function $\calL(\theta;D)=\frac{1}{n}\sum\limits_{i=1}^n\ell(\theta;d_i)$ 
defined over a data set $D$, a constraint set $\calC$, and a threshold $\threshold$, an algorithm $\mathcal A$ is $(\lambda, \psi, \gamma)$-sampler for the Rashomon set $$\rashomonemp=\left\{\theta \in \calC:\calL(\theta;D)\leq\min\limits_{\theta^*\in\calC}\calL(\theta^*;D)+\threshold\right\}$$ if distribution of $\privT \gets \mathcal A$ is $\tvd$-far in total variation distance (TVD) from a distribution $\pi$ such that 
\begin{enumerate}
    \item \textbf{Uniform sampling:} The marginal distribution of $\pi$ over $\rashomonemp$ is uniform.
    \item \textbf{Maximality condition:} For any $\theta \in \calG$ and $\theta' \notin \calG$, the distribution $\pi$ assigns probability density to $\theta$ which is greater than or equal to that of $\theta'$. 
    \item \textbf{Excess empirical risk guarantee:} $\mathbb{E}_\calA\left[\calL(\privT;D)\right]\leq \min\limits_{\theta^*\in\calC}\calL(\theta^*;D)+\threshold+\slack.$
\end{enumerate}
\end{defn}
The \emph{maximality} condition rules out trivial and uninteresting solutions. It ensures that (ignoring privacy constraints), it is possible to combine the Rashomon sampler with rejection sampling to efficiently get an uniform sample from $\calG$. In particular, we show in Theorem~\ref{thm:hitting} that, if $\threshold = \omega(p \gamma)$, then the Gibbs sampler has $1-o(1)$ probability of hitting the Rashomon set for the values of $\beta$ stated in Theorem~\ref{infthm:LD}. 
We presented Definition~\ref{def:privatesampler} with respect to empirical accuracy guarantee for the ease of presentation. One can also define \emph{excess population risk guarantee}: $$\mathbb{E}_{\calA,d\sim\dist}\left[\ell(\privT;d)\right]\leq \threshold+\min\limits_{\theta\in\calC}\mathbb{E}_{d\sim\dist}\left[\ell(\theta;d)\right]+\slack',$$ 
where $\slack'$ is the {\em population level slack} and $\dist$ is the distribution from which the data samples in the data set $D$ are drawn i.i.d. 

In Theorems~\ref{thm:expMech-rashomon} and~\ref{thm:expMech-rashomon-approx}, we show that our LD based Gibbs sampler from Theorem~\ref{infthm:LD} is a $(0,\threshold,p/\beta)$-Rashomon sampler. Furthermore, in Theorem~\ref{thm:rashomon-population}, we show that the excess population risk (i.e., $\threshold+\gamma'$) for the Rashomon set sampler, is bounded by $\frac{p}{\beta}+O\left(\threshold+L\sqrt{\threshold/m}+L^2/mn\right)$. We give a summary of the bounds in Table~\ref{tb:resultsrashomon}.

We next discuss the use case of predictive multiplicity eluded above. Consider a classification problem with $c$ classes {and $\Delta_c$ be the probability simplex}. Let $(y,\bfx)$ be a test sample and let $f(\bfx;\theta)\in\Delta_c$ be the prediction function that provides a probability distribution across the $c$ classes. The objective is to estimate the variance,  $\mathsf{Var}(f(\bfx,\theta))$, where $\theta\sim_{\sf unif}\rashomonemp$. One can get its approximate estimate by sampling $\left\{\theta_1,\ldots,\theta_k\right\}\sim_{\sf iid}\rashomonemp$, and estimating the standard deviation of $\{f(\bfx,\theta_1),\ldots,f(\bfx,\theta_k)\}$. While there is a privacy cost in sampling $k$ models,~\cite[Section 6.4]{brawner2018bootstrap} shows that the variance estimation comes at~\emph{no additional privacy cost} when compared to that of outputting a single model. Given the standard deviation, one can then decide on the class to predict for $\bfx$, either via further randomization, or other strategy, including more sophisticated measures like {\em Rashomon capacity}~\citep{hsurashomon}. We leave exploring their DP variants for future research.

\mypar{Organization} For the ease of presentation and owing to the generality of the Rashomon set sampling problem, we state all our main results in that context in Section~\ref{sec:rashomonSampling}. In Section~\ref{sec:sgld} we provide the details for the discretization of the LD algorithm. In Section~\ref{sec:DPERMSCO} we provide DP-ERM/DP-SCO results obtained by setting $\threshold=0$ for the Rashomon set sampler. Finally, in Section~\ref{sec:discussion} we end with discussions and future directions. We enumerate our notations in \Cref{tb:notation}.

{\color{blue}
\subsection{Omitted results from v3}\label{sec:omitted}

Most results in the previous version of this paper on arXiv (v3) appear in this version or are subsumed by results in this version. For ease of presentation, a few results in v3 were not carried over. The following is a complete list of the omitted results:
\begin{itemize}
    \item Section 3.3 of v3, which gives a tighter analysis of the empirical loss guarantees for the continuous exponential mechanism on non-convex losses than the one in Bassily et al. \cite{BST14} by removing the ``small ball' assumption.
    \item Section 5 and Appendix F of v3 analyze the DP-ERM/SCO guarantees of finite-time LD under approximate DP by a gradient descent-like analysis (as opposed to the results in this version, which use the analysis of the exponential mechanism).
    \item Section 7 of v3 (i) gives intuition for why the sum of step sizes in DP-SGD can be quite small (ii) shows that at some time when DP-SGD/LD achieve the asymptotically optimal ERM bound, their output distribution is total variation distance $1-o(1)$ from their stationary distribution. Note that (ii) in v3 is shown for a loss which is not strongly convex, i.e. (ii) does not contradict the results in this paper which rederive the optimal ERM bound by showing the chain mixes for a strongly convex loss.
    \item Section 8 of v3 shows a bound on the last-iterate R\'enyi divergence between running DP-LD on two adjacent databases that does not go to infinity as $t$ goes to infinity. This is somewhat subsumed by the analysis in~\cref{sec:rashomonSampling}, which shows a qualitatively similar statement in terms of approximate DP instead of R\'enyi DP.
\end{itemize}
}
\vspace{-3mm}
\section{Rashomon Set Sampling}
\label{sec:rashomonSampling}
In this section, we provide the privacy analysis (Sections~\ref{sec:RPrivacy} and~\ref{app:simple})\footnote{The analysis in Section~\ref{sec:RPrivacy} is via analyzing the privacy of the path of the LD that converges to the Gibbs distribution, and an alternate one in Section~\ref{app:simple} via directly analyzing the privacy of the Gibbs distribution.}, and the utility analysis (Sections~\ref{sec:empUtil} and~\ref{sec:roshPOP}) for the Rashomon set sampler based on the Gibbs distribution proportional to $\exp(-\beta \widetilde{\calL}(\theta;D))$, where $$\widetilde{\calL}(\theta;D):= \max\{\calL(\theta;D), \threshold + \min\limits_{\theta^*\in\calC}\calL(\theta^*;D)\}.$$ 

Both the privacy and the utility guarantee primarily depend on the choice of the inverse temperature $\beta$. Under $(\epsilon,\delta)$-DP, we can operate with higher values of $\beta$ than in the case of $\epsilon$-DP (see Theorems~\ref{thm:expMech-rashomon} and~\ref{thm:expMech-rashomon-approx}): this translates to a better utility for certain choices of  $\threshold$. In Table~\ref{tb:resultsrashomon} we provide a summary of the empirical/population risk guarantees for our Rashomon samplers satisfying a given privacy constraint.

In our algorithmic version of the Gibbs distribution, we instantiate it with the LD described in~\Cref{fig:LD}. The LD that is used to instantiate the Gibbs distribution in the $\epsilon$-DP case unfortunately requires running the algorithm for $t\to\infty$ to reach within $\epsilon$ of the stationary distribution in terms of $\ell_\infty$-distance\footnote{\cite{BST14,mangoubi2021sampling} provides rejection sampling based polytime algorithms, but  lack the generalization properties of LD.}. However, if we are willing to tolerate $(\epsilon,\delta)$-DP guarantee, then the LD can be run in time $\approx \frac{\log(1/\delta)}{\beta m}$, where $m$ is the strong convexity parameter. A standard discretization of the LD we use in this section (via SGLD), that makes the algorithm run on a finite precision machine, can be found in Section~\ref{sec:sgld}. All the missing proofs of this section appear in \Cref{sec:deferredsectionrashomonsampling}.

\begin{table}[ht]
\begin{center}
\begin{tabular}{|c|c|c|}
\hline
Privacy guarantee & $\epsilon$-DP & $(\epsilon,\delta)$-DP\\
\hline
Excess empirical risk ($\gamma$)   &   $\frac{Lp}{\epsilon n}$        &   $\frac{L^2p \log(1/\delta)}{m\epsilon^2 n^2}$          \\
\hline 
Excess population risk & $\frac{Lp}{\epsilon n} + \threshold + L \sqrt{\frac{\threshold}{m}} + \frac{L^2}{m n}$      &     
$\frac{L^2p\log(1/\delta)}{m\epsilon^2 n^2} + \threshold + L \sqrt{\frac{\threshold}{m}} + \frac{L^2}{m n}$            \\
\hline
\end{tabular}
\end{center}
\caption{Summary of our Rashomon sampler guarantees. 
In all results, $\lambda$ (the sampling error) is 0. In bounds where the parameters appear, we assume $L$-Lipschitzness, $m$-strong convexity, and $M$-smoothness within $\calC$. }
\label{tb:resultsrashomon}
\vspace{-5mm}
\end{table}

\subsection{Privacy Guarantees and Running Time}
\label{sec:RPrivacy}
The privacy in $\epsilon$-DP case follows from the fact that $\max\{\calL(\theta; \dataset), \min\limits_{\theta^*\in\calC}\calL(\theta^*; \dataset) + \threshold\}$ cannot change by more than $L\ltwo{\calC}/n$ when we change one data point because each $\ell\in[0, L \ltwo{\calC}]$. 

\begin{thm}
[$\epsilon$-DP sampler; Theorem 3.1 of \cite{BST14}]
Suppose we have a constraint set $\calC$ and a loss function $\ell(\theta; d)$ such that for all $d$, $\ell(\cdot; d)\in\mathbb{R}^+$, and is $L$-Lipschitz within $\calC$. Then, sampling $\privT$ from the Gibbs distribution $\exp(-\beta \max\{\calL(\theta; \dataset), \min\limits_{\theta^* \in \calC}\calL(\theta^*;D)+\threshold\})$ is $\epsilon$-differentially private for $\beta = O(\frac{\epsilon n}{ L \ltwo{\calC}})$ and all $\threshold$.
\label{thm:expMech-rashomon}
\end{thm}

\noindent \textbf{$(\epsilon,\delta)$-DP Sampler:} If $\ell(\theta;\cdot)$ is $m$-strongly convex  and our goal is $(\epsilon,\delta)$-DP, we can use larger values of $\beta$. However, for the settings when $\threshold>0$, we would additionally require $\ell(\theta;\cdot)$ to be $M$-smooth. 
We first need a ``continuous'' composition theorem that bounds the R\'enyi divergence between two instances of LDs run on adjacent databases:

\begin{lem}\label{lem:renyifinite}
Let $\theta_0, \theta_0'$ have the same distribution $\Theta_0$, $\theta_t$ be the solution to \eqref{eq:PLD} given $\theta_0$ and data set $\dataset$ (and correspondingly $\theta'_0$ and $\theta'_t$ for a data set $D'$). Let $\Theta_{[0,t]}$ ({$\Theta'_{[0,t]}$}) be the distribution of the trajectory of LD $\{\theta_{t'}\}_{t' \in [0,t]}$ ({$\{\theta'_{t'}\}_{t' \in [0,t]}$}, respectively). Suppose we have that $\ltwo{\grad \calL(\theta; \dataset) - \grad \calL(\theta; \dataset')} \leq \Delta$ for all $\theta$. Then $\forall\alpha \geq 1$:
$$R_\alpha(\Theta_{[0,t]}, \Theta_{[0,t]}') \leq \frac{\alpha \beta^2 \Delta^2 t}{4}.$$
\end{lem}
The idea behind the proof is to use a bound on the divergence between Gaussians and RDP composition to provide a bound on the divergence between the projected noisy gradient descents on datasets $\dataset$ and $\dataset'$. Then, taking the limit as the step size in gradient descent goes to 0 and applying Fatou's lemma (Lemma~\ref{lem:Fatou}), we get the bound above. A full proof is deferred to Appendix~\ref{app:langevin-deferred}. R\'enyi divergence bounds imply $(\epsilon, \delta)$-DP privacy guarantees (Fact~\ref{fact:renyitoapx}), which we use to prove the following theorem in Appendix~\ref{sec:rashomon-privacy-proof}:

 {
\begin{thm}
[$(\epsilon,\delta)$-DP sampler]
\label{thm:expMech-rashomon-approx}
Suppose we have a constraint set $\calC$ and a loss function $\ell(\theta; d)$ such that for all $d$, $\ell(\cdot; d)\in\mathbb{R}^+$ is $m$-strongly convex, $M$-smooth,  and is $L$-Lipschitz within $\calC$. For $t = \tilde{O}_\delta\left(\frac{1}{\beta m}\right)$ and an appropriate choice of $\theta_0$, let $\Theta_t$ be the distribution over $\theta_t$ given by running \eqref{eq:PLD} on $\max\{\calL(\theta; \dataset), \min\limits_{\theta^*\in\calC}\calL(\theta^*;D)+\threshold\}$. Then $\Theta_t$ is within total variation distance $\delta$ of the Gibbs distribution on $\max\{\calL(\theta; \dataset), \min\limits_{\theta^*\in\calC}\calL(\theta^*;D)+\threshold\}$, and is $(\epsilon, \delta)$-DP if any of the following holds:
\begin{enumerate}
    \item [(i)] $\threshold\in\left(\frac{2L^2}{Mn^2},\frac{L\ltwo{\calC}}{2}\right]$, and $\beta = \widetilde{\Theta}\left(\frac{\epsilon^2 (m/M)}{ \log(L\ltwo{\calC}/\delta^2)\log(1/\delta)}\cdot \frac{1}{\threshold} \right)$, 
    \item [(ii)] $\threshold \leq \frac{2L^2}{Mn^2}$ and $\beta = \widetilde{\Theta}\left(\frac{\epsilon^2 n^2 m }{L^2 \log(L\ltwo{\calC}/\delta^2)\log(1/\delta)} \right)$. 
\end{enumerate}
Furthermore, the privacy holds even if we release the entire trajectory $\{\Theta_{t'}\}_{t' \in [0, t]}$, and even without convexity.
\end{thm}

If $\threshold$ has a dependence $o(1/n)$ on $n$, this gives a better dependence of $\beta$ on $n$ than Theorem~\ref{thm:expMech-rashomon}. For privacy, we first bound the sensitivity of the thresholded loss by $O(\max\{L/n, \sqrt{M \threshold}\})$. We then appeal to  
 Lemma~\ref{lem:renyifinite} and the translation from R\'enyi divergence bounds to $(\epsilon, \delta)$-DP bounds (Fact~\ref{fact:renyitoapx}). Since Lemma~\ref{lem:renyifinite} does not require convexity and allows for releasing the entire trajectory, the same is true for the privacy guarantee of  Theorem~\ref{thm:expMech-rashomon-approx}.
The sampling guarantee is given by showing that the Gibbs distribution satisfies log-Sobolev inequality (LSI; see Definition~\ref{def:LSI}), a measure of concentration. This implies convergence of LD to the Gibbs distribution by results in \cite{vempala2019renyi}. A few comments are in order about the theorem.

\paragraph{Direct analysis of Gibbs distribution}
Note that by triangle inequality, Theorem~\ref{thm:expMech-rashomon-approx} also implies that, under strong convexity, the Gibbs distribution is $(\epsilon, 2\delta)$-DP. 
We can also analyze the Gibbs distribution directly: Since the Gibbs distribution satisfies LSI, one can obtain an isoperimetric inequality for the probability measure via~\cite{ledoux1999concentration}. Using this isoperimetric inequality and the coupling between Gibbs distributions that we later state in  Theorem~\ref{thm:rashomon-uniform}, one can obtain a bound on $\beta$ that improves Theorem~\ref{thm:expMech-rashomon-approx} by log factors, giving tight bounds in the case $\threshold = O\left( \frac{L^2}{Mn^2}\right)$. However, this privacy proof relies heavily on the convexity and does not give an ``anytime'' private sampler like our LD-based proof. For the sake of completeness, we provide a proof of this result in Section~\ref{app:simple} (see Theorem~\ref{thm:GLLRashomon} for a precise statement).}

\paragraph{Dependence of $\beta$ on $\threshold$} 
We believe the relation $\beta \propto \min\{1 / \threshold, n^2\}$ in Theorem~\ref{thm:expMech-rashomon-approx} is necessary. 
To see this, consider mean estimation on an all zero databases $D$ and on neighboring $D'$ that contain just a single 1. Fix $\epsilon = 1$ for simplicity. The Gibbs samplers for these datasets have densities proportional to $\exp(-\beta \min\{\ltwo{x}^2/2, \threshold\})$ and $\exp(-\beta \min\{\ltwo{x - 1/n}^2/2, \threshold\})$, respectively. For the case $\threshold = 0$, this is just a Gaussian mechanism with sensitivity $1/n$ and variance $1/{\beta}$, and one way to argue this mechanism is differentially private is to provide a tail bound on $x$, which gives a tail bound on the privacy loss. For $\threshold > 0$, the unnormalized density decreases by a factor of $\exp(-\Omega(\beta \threshold))$ in the range $[-\sqrt{2\threshold}, \sqrt{2\threshold}]$ when we apply thresholding to the loss function. In turn, if most of the probability mass is in this interval, a tail bound on events outside this interval can get worse by a factor of $\exp(-\Omega(\beta \threshold))$.  Suppose we use $\beta \approx n^2$, which gives $(1, \delta)$-DP for $\threshold = 0$. Then the interval $[-\sqrt{2\threshold}, \sqrt{2\threshold}]$ contains all points within $\sqrt{\threshold/\beta} \approx n \sqrt{\threshold}$ standard deviations of the mean. So this interval contains most of the probability mass of the mechanism when $\threshold = \widetilde{\Omega}(1/n^2)$. This roughly matches the ``transition point'' in Theorem~\ref{thm:expMech-rashomon-approx}.  In order for a tail bound that holds w.p. $1-\delta$ on the Gaussian mechanism to be non-vacuous after thresholding, we need $\beta \threshold = O(\log(1/\delta))$. This roughly matches our choice of $\beta$ after the ``transition point.'' 

\paragraph{Running time} In the following discussion, we will assume $L=\ltwo{\calC}=\Theta(1)$ for brevity. The LD for the $\epsilon$-DP sampler (mentioned in Theorem~\ref{thm:expMech-rashomon}) runs for time $t\to\infty$ to obtain the privacy/utility trade-off obtained by Theorems~\ref{thm:expMech-rashomon} and~\ref{thm:expConv-rashomon-util}. However, if we are willing to tolerate $(\epsilon,\delta)$-DP, then assuming the loss function $\calL(\theta;D)$ is $m$-strongly convex, one can obtain asymptotically the same privacy/utility trade-off, and run for time $t_\beta=\widetilde{O}_{\delta}\left(\max\left\{\frac{1}{\beta m},\frac{\threshold}{m}\right\}\right)$. This  follows from  Lemma~\ref{lem:rashomon-convergence}. Setting $\beta=\epsilon n$ from Theorem~\ref{thm:expMech-rashomon}, we have $t_\beta=\widetilde{O}_{\delta}\left(\max\left\{\frac{1}{m\epsilon n},\frac{\threshold}{m}\right\}\right)$. The $(\epsilon,\delta)$-DP sampler from Theorem~\ref{thm:expMech-rashomon-approx} also runs in time $t_\beta$, where $\beta$ is chosen based on Theorem~\ref{thm:expMech-rashomon-approx}. Hence, to obtain the best $(\epsilon,\delta)$-DP Rashomon sampler, one needs $t_\beta=\widetilde{O}_{\delta}\left(\max\left\{\frac{1}{m\epsilon^2 n^2},\frac{\epsilon^2 M\threshold}{m}\right\}\right)$. In Section~\ref{sec:sgld}, we discuss how one can obtain an approximate sampler using only discrete noisy gradient steps.

\paragraph{Probability of hitting the Rashomon Set}
The Gibbs measure induced by the LD algorithm i.e., \[\exp(-\beta \max\{\calL(\theta; \dataset), \threshold + \min_{\theta \in \calC} \calL(\theta;\dataset)\})\]  trivially satisfies the maximality condition for the Rashomon sampler. In Theorem~\ref{thm:hitting}, we show that the $\epsilon$-DP Rashomon set sampler hits the set $\calG$ with probability at least  $1-o(1)$, if $\threshold=\widetilde{\omega}(p^2/(\epsilon n))$. (Here, we assumed all other parameters to be constant.)  Although all our Rashomon set samplers  are forced to provide non-zero probability mass on the Rashomon set due to the maximality condition, it is unclear how to obtain a similar guarantee as Theorem~\ref{thm:hitting} for our $(\epsilon,\delta)$-DP sampler. We leave it as an open question.

\subsection{Excess Empirical Risk Guarantees}
\label{sec:empUtil}

The excess empirical risk guarantee follows from Theorem 3.2 in \cite{BST14}:

\begin{thm}
\label{thm:expConv-rashomon-util}
Assume each loss function is convex. Then sampling $\privT$ from the Gibbs distribution \[\exp(-\beta \max\{\calL(\theta; \dataset),\min\limits_{\theta^*\in\calC} \calL(\theta^*; \dataset) + \threshold\})\] is a $(0, \threshold, p/\beta)$-Rashomon sampler. 
\end{thm}

Using Theorem~\ref{thm:expConv-rashomon-util}, to get the best risk bound it suffices to find the largest value of $\beta$ that preserves DP under given assumptions on the loss function.  Theorem~\ref{thm:expMech-rashomon-approx} suggest that when $\threshold=0$, the setting of $\beta$ that is required for $(\epsilon,\delta)$-DP is  independent of the smoothness parameter $M$, and, hence can also be applied to non-smooth functions. Combining Theorems~\ref{thm:expMech-rashomon},~\ref{thm:expMech-rashomon-approx}, and~\ref{thm:expConv-rashomon-util} we get the existence of the following three Rashomon samplers: 

\begin{itemize}
    \item $\left(0, \threshold,\widetilde{O}\left(p\cdot\frac{L \ltwo{\calC}}{\epsilon n}\right)\right)$-sampler for all $\threshold$ and is $\epsilon$-DP.
    \item $\left(0, \threshold,\widetilde{O}\left(p\threshold\cdot\frac{M \log(L\ltwo{\calC}/\delta^2)\log(1/\delta)}{m\epsilon^2}\right)\right)$-sampler when $\threshold\in\left(\frac{2L^2}{Mn^2},\frac{L\ltwo{\calC}}{2}\right]$ and is $(\epsilon,\delta)$-DP.
    \item $\left(0,\threshold,\widetilde{O}\left(p\cdot\frac{L^2 \log(L\ltwo{\calC}/\delta^2)\log(1/\delta)}{m\epsilon^2n^2}\right)\right)$-sampler when $\threshold \leq \frac{2L^2}{Mn^2}$ and is $(\epsilon,\delta)$-DP. 
\end{itemize}

In the $(\epsilon,\delta)$-DP setting if instead of sampling from the Gibbs distribution, we operate with the LD, then in the Rashomon sampler we set $\lambda=\delta$ instead of $\lambda=0$ as done until now. The loss guarantee worsens by at most $O(\delta L^2 / m)$ if we use LD instead of the Gibbs sampler, so the earlier bounds remain unchanged for $\delta = O(1/\epsilon^2 n^2)$.

Ignoring the polylogarithmic terms, if $\threshold=\widetilde{\Omega}\left(\frac{\epsilon L\ltwo{\calC}}{(M/m)}\cdot\frac{1}{n}\right)$, then our privacy analysis of the $\epsilon$-DP Rashomon sampler gives a better bound on $\beta$ than our analysis of the $(\epsilon,\delta)$-DP sampler, and vice versa when $\threshold=\widetilde{o}\left(\frac{\epsilon L\ltwo{\calC}}{(M/m)}\cdot\frac{1}{n}\right)$. As previously mentioned, we believe that the weaker bound on $\beta$ for higher values of $\threshold$ in our analysis of the $(\epsilon, \delta)$-DP sampler is fundamental to the problem. 

\subsection{Excess Population Risk Guarantees}
\label{sec:roshPOP}

For the same Rashomon samplers, we can derive bounds on their population risks under strong convexity. We give a bound on uniform stability (see Definition~\ref{def:unifStab}) of the Gibbs distribution:

\begin{thm}\label{thm:rashomon-uniform}
Suppose we have $\calC$ and $\ell(\cdot; d)\in\mathbb{R}^+$ such that, for all $d$, $\ell(\cdot; d)$ is $m$-strongly convex, and is $L$-Lipschitz within $\calC$. Then sampling $\privT$ from the Gibbs distribution proportional to \[\exp(-\beta \max\{\calL(\theta; \dataset), \min\limits_{\theta^*\in\calC}\calL(\theta^*;D) + \threshold\})\] satisfies $\left(4 L \sqrt{\frac{2\threshold}{m}} + \frac{2 L^2}{m n}\right)$-uniform stability.
\end{thm}

\begin{proof}
Let $\theta_t, \theta'_t$ be the solutions to running \eqref{eq:PLD} from the same initialization on $\calL(\theta; D)$ and $\calL(\theta; D')$, respectively. Similarly, let $\tilde{\theta}_t $ be the solutions to running \cref{eq:PLD} on $\max \left\{\threshold + \min\limits_{\theta^*\in\calC}\calL(\theta^*; D), \calL(\theta; D) \right\}$ and $\tilde{\theta}'_t$ be the solutions to running \cref{eq:PLD} on $\max \left\{\threshold + \min\limits_{\theta^*\in\calC}\calL(\theta^*; D'), \calL(\theta; D') \right\}$. Let $W_\infty$ denote the the $\infty$-Wasserstein distance. Then 
we will show that for all $t$, $W_\infty(\tilde{\theta}_t, \tilde{\theta}'_t) \leq 4 \sqrt{\frac{2 \threshold}{m}} + \frac{2L}{mn}$. 
Taking the limit as $t$ goes to infinity, we get a $\infty$-Wasserstein distance bound between the Gibbs samplers. The theorem now follows by using $L$-Lipschitzness.

We now show the desired bound on $W_\infty(\tilde{\theta}_t, \tilde{\theta}'_t)$. By the triangle inequality, we have 
$$W_\infty(\widetilde \theta_t, \widetilde \theta_t') \leq W_\infty(\tilde{\theta}_t, \theta_t) + W_\infty(\tilde{\theta}'_t, \theta'_t) + W_\infty(\theta_t, \theta'_t).$$ 

So, it suffices to prove that for all $t$, 
\[W_\infty(\tilde{\theta}_t, \theta_t) \leq 2 \sqrt{\frac{2 \threshold}{m}},\quad W_\infty(\tilde{\theta}'_t, \theta'_t) \leq 2 \sqrt{\frac{2 \threshold}{m}}, \quad \text{and} \quad  W_\infty(\theta_t, \theta'_t) \leq \frac{2L}{mn}.\]

We first prove the desired bound on $W_\infty(\tilde{\theta}_t, \theta_t)$, the bound on $W_\infty(\tilde{\theta}'_t, \theta'_t)$ follows identically. 

\mypar{Bounding $W_\infty(\tilde{\theta}_t, \theta_t)$}
We show that conditioned on the value of a \textit{shared} Brownian motion $W_t$, $\ltwo{\theta_t - \tilde{\theta}_t} \leq 2 \sqrt{\frac{2 \threshold}{m}}$ holds deterministically, which implies the desired Wasserstein distance bound. We split $[0, \infty)$ into intervals of maximal length for which one of the following three holds throughout each interval: 
\begin{enumerate}
    \item [(i)] $\tilde{\theta}_t \notin \calG$,
    \item [(ii)] $\tilde{\theta}_t \in \calG, \theta_t \in \calG$, and 
    \item [(iii)]  $\tilde{\theta}_t \in \calG, \theta_t \notin \calG$.
\end{enumerate}
     
In case (ii), by strong convexity throughout the interval we have $$\ltwo{\theta_t - \tilde{\theta}_t} \leq \ltwo{\calG} \leq 2 \sqrt{\frac{2 \threshold}{m}}.$$ 

So it suffices to show that in cases (i) and (iii), $\ltwo{\theta_t - \tilde{\theta}_t}$ is non-increasing throughout the interval. Then the desired Wasserstein distance bound holds by induction, since initially $\theta_0 = \tilde{\theta}_0$.

In case (i), since $\calL$ is convex, projection is contractive (which implies $\frac{d \ltwo{\theta_t - \tilde{\theta}_t}^2}{dt}$ can only increase if we use \eqref{eq:LD} instead of \eqref{eq:PLD}), and that we are using a shared Brownian motion $W_t$, we have 

\begin{align*}
\frac{1}{2} \cdot \frac{d}{dt} \ltwo{\theta_t - \tilde{\theta}_t}^2 &\leq \left\langle \frac{d \theta_t}{dt} - \frac{d \tilde{\theta}_t}{dt}, \theta_t - \tilde{\theta}_t \right\rangle = \beta \langle -(\nabla \calL(\theta_t; D) - \nabla \calL(\theta_t'; D)), \theta_t - \tilde{\theta}_t  \rangle \leq 0 .\end{align*}

Similarly, in case (iii), by convexity and since $\tilde{\theta}_t$ is in the Rashomon set and $\theta_t$ is not (or $\calL(\tilde{\theta}_t;D) \leq \calL({\theta}_t;D)$), we have 

\begin{align*}
\frac{1}{2} \cdot  \frac{d \ltwo{\theta_t - \tilde{\theta}_t}^2}{dt} & \leq \left\langle \frac{d \theta_t}{dt} - \frac{d \tilde{\theta}_t}{dt}, \theta_t - \tilde{\theta}_t \right\rangle = \beta \left\langle \nabla \calL(\theta_t; D), \tilde{\theta}_t - \theta_t  \right\rangle \leq \beta (\calL(\tilde{\theta}_t; D) - \calL(\theta_t; D)) \leq 0.\end{align*}

\mypar{Bounding $W_\infty(\theta_t, \theta'_t)$} We again show that conditioned on the value of a \textit{shared} Brownian motion $W_t$, $\ltwo{\theta_t - \theta'_t} \leq \frac{2L}{mn}$ holds deterministically. We again use the fact that projection is contractive, so we can consider using \eqref{eq:LD} instead of \eqref{eq:PLD}. Then by $m$-strong convexity and $L$-Lipschitzness:

\begin{align}
\frac{1}{2} \cdot \frac{d \ltwo{\theta_t - \theta'_t}^2}{dt} &\leq \left\langle \frac{d \theta_t}{dt} - \frac{d \theta'_t}{dt}, \theta_t - \theta'_t \right\rangle = - \beta \langle \nabla \calL(\theta_t; D) - \nabla(\calL(\theta_t', D'), \theta_t - \theta'_t \rangle \nonumber \\
&= - \beta \left(\langle \nabla \calL(\theta_t; D) - \nabla(\calL(\theta_t', D), \theta_t - \theta'_t \rangle + \langle \nabla \calL(\theta_t'; D) - \nabla(\calL(\theta_t', D'), \theta_t - \theta'_t \rangle \right) \nonumber \\
&\leq \beta {\left(-m \ltwo{\theta_t - \theta_t'}^2 + \frac{2L}{n} \ltwo{\theta_t - \theta_t'}\right)}. \label{eq:distancebound}
\end{align}

Then, if $\ltwo{\theta_t - \theta_t'} > \frac{2L}{mn}$, then we have $(\ref{eq:distancebound})<0$. That is, we have $\frac{d \ltwo{\theta_t - \theta'_t}^2}{dt} <0$. This implies the desired Wasserstein distance bound completing the proof of Theorem~\ref{thm:rashomon-uniform}.
\end{proof}

\begin{remark}
We note that in the second part of the proof, one can instead take the $L^2 / mn$-uniform stability of (noisy) gradient descent on strongly convex losses proved in \cite{HardtRS16}, and then take the limit as $\eta \rightarrow 0, T \rightarrow \infty$ to conclude the same uniform stability bound holds for the Gibbs distribution, rather than appeal to the derivative of the distance between $\theta_t$ and $\theta_t'$.
\end{remark}

It is straightforward to see that a $(0, \threshold, \gamma)$-Rashomon sampler has expected excess empirical risk at most $\threshold + \gamma$. Then, an algorithm's excess population risk is at most its excess empirical risk plus its uniform stability (see Lemma~\ref{lem:emptopop-us}), giving us the following result:

\vspace{0mm}
\begin{thm}\label{thm:rashomon-population}
Assume that each of the individual loss function $\ell(\theta;\cdot)\in\mathbb{R}^{+}$ is $m$-strongly convex and $L$-Lipschitz within the constraint set $\calC$. Then

\begin{itemize}
    \item There exists an $\epsilon$-DP rashomon sampler with threshold $\threshold$ and excess population risk \[O\left(\frac{L\ltwo{\calC}p}{\epsilon n} + \threshold + L \sqrt{\frac{\threshold}{m}} + \frac{L^2}{m n}\right).\]
    \vspace{-3mm}
    \item There exists an $(\epsilon, \delta)$-DP rashomon sampler with threshold $\threshold$ and excess population risk (assuming $M$-smoothness)
    \[\widetilde{O}\left(p \max\left\{\threshold, \frac{L^2}{Mn^2}\right\}\cdot\frac{M \log(L\ltwo{\calC}/\delta^2)\log(1/\delta)}{m\epsilon^2}+\threshold + L \sqrt{\frac{\threshold}{m}} + \frac{L^2}{m n}\right).\]
\end{itemize}
\end{thm}

\subsection{Interlude: Proof of Privacy for the Gibbs Sampler via Isoperimetry}\label{app:simple}
\begin{thm}
\label{thm:simple}
Suppose $\ell$ is $m$-strongly convex, $M$ smooth, and $L$-Lipschitz within $\calC$.
Let \[\calL'(\theta; D) := \max\{\calL(\theta; D), \min\limits_{\theta^* \in \calC} \calL(\theta^*; D) + \threshold\}.\] Then for
$\beta = O\left(\frac{\min\{\epsilon^2, \epsilon\} m}{\max\{L^2 / n^2, M \threshold \} \log(1/\delta)}\right),$
sampling from the distribution with density proportional to $\exp(-\beta \calL'(\theta; D)) \cdot \mathds{1}(\theta \in \calC)$ satisfies $(\epsilon, \delta)$-DP.
\label{thm:GLLRashomon}
\end{thm}

\begin{proof}
    Let $P$ be the probability measure for the distribution induced by $D$ and $Q$ be the same for $D'$. Let $g(\theta) = \log(P(\theta)/Q(\theta))$. Then $g$ is a $6 \beta \max\{L / n, \sqrt{M \threshold}\}$-Lipschitz function. It suffices to show that
    $\Pr_{\theta \sim P}\left[ g(\theta) > \epsilon \right] \leq \delta$.
    We first bound $E_{\theta \sim P}\left[g(\theta)\right]$ as a function of $\beta$. This is simply the KL divergence between $P, Q$, which must be non-negative. By symmetry, $E_{\theta \sim Q}\left[g(\theta)\right]$ is non-positive. In addition, we can bound the difference between these two expressions: in the proof of Theorem~\ref{thm:rashomon-uniform}, we showed that the $\infty$-Wasserstein distance between $P$ and $Q$ is at most $4\sqrt{\frac{2\threshold}{m}} + \frac{2L}{mn}$. Then, by Lipschitzness of $g$,
    \begin{align}
    E_{\theta \sim P}\left[g(\theta)\right] &\leq E_{\theta \sim Q}\left[g(\theta)\right] + 6 \beta \max\{L / n, \sqrt{M \threshold}\} \cdot \left(4\sqrt{\frac{2\threshold}{m}} + \frac{2L}{mn}\right)\nonumber\\
    &  \leq 6 \beta \max\{L / n, \sqrt{M \threshold}\} \cdot \left(4\sqrt{\frac{2\threshold}{m}} + \frac{2L}{mn}\right)
    \\
    &\leq 36 \sqrt{2} \beta \max\left\{\frac{L^2}{mn^2}, \sqrt{\frac{M}{m} }\threshold\right\} .\label{eq:privlossexp}
    \end{align}
    
    We next give a high probability bound on $g$ as a function of $\beta$. $P$ satisfies LSI with constant $\beta m \exp(- \beta \threshold)$ by Lemma~\ref{lem:rashomon-lsi}. Then by Proposition 2.3 in \cite{ledoux1999concentration}, plugging in our LSI constant and Lipschitzness bound for $g$ we have:
    
    \begin{equation}\label{eq:privlosstail}
    \Pr_{\theta \sim P}[g(\theta) >  \mathbb{E}_{\theta \sim P}[g(\theta)] + \beta r] \leq \exp\left(-\frac{\beta m \exp(- \beta \threshold) r^2}{18 \max\{L^2 / n^2, M \threshold / 2\}}\right)
    \end{equation}
    
    Setting 
    \[r = \frac{3\sqrt{2} \max\{L/n, \sqrt{M \threshold/2}\}}{\sqrt{\beta m} \exp(-\beta \threshold / 2)}\sqrt{\log(1/\delta)},\] the right hand side in \cref{eq:privlosstail} becomes $\delta$. Setting
    \[\beta \leq \frac{\min\{\epsilon^2, \epsilon\} m}{c \max\{L^2 / n^2, M \threshold\} \log(1/\delta)},\]
    where $c$ is a sufficiently large constant, gives that the upper bound in \eqref{eq:privlossexp} is at most $\epsilon / 2$ and $\beta r \leq \epsilon / 2$. Here, we use that for this choice of $\beta, \beta \threshold \leq 1$. Plugging in to \eqref{eq:privlosstail} we get $\Pr_{\theta \sim P}\left[g(\theta) > \epsilon \right] \leq \delta$ exactly as desired. This completes the proof of Theorem~\ref{thm:simple}.
\end{proof}
\vspace{-5mm}
\subsection{Comparison with Contemporary Work of~\cite{gopi2022private}}
\label{sec:gopietal}

All our results for DP-ERM and DP-SCO, and their relation with Langevin diffusion is contemporary to~\cite{gopi2022private} (which is also formally acknowledged by~\cite{gopi2022private}). Our results specific to Rashomon sets  (i.e., with setting $\threshold>0$) and SGLD are subsequent to their work. In the following, we discuss the technical differences between two works, and highlight the settings in which one might be better over the other.

\paragraph{On the privacy aspect} Since~\cite{gopi2022private} gives Gaussian DP guarantees, which do not translate to $\epsilon$-DP guarantees, we will restrict the discussion specific to $(\epsilon,\delta)$-DP. 
Unlike \cite{gopi2022private}, our privacy guarantee (for LD as well as SGLD) is independent of convexity. This is highly desirable for broader applicability in settings where where the loss function may not be globally convex, but has local convexity properties (e.g., losses emerging from deep learning settings~\citep{izmailov2018averaging,keskar2016large}.)
 Our privacy holds for the entire path of the optimization. In contrast, \cite{gopi2022private} only guarantee privacy for the final model release. In particular, this implies that we can stop and output the model at any point of the optimization trajectory. While this might not yield optimal model, the privacy is never compromised. The ability to release the  path has been used in subsequent works~\citep{shejwalkar2022recycling,rabanser2023training} for uncertainty quantification.

\paragraph{On the proof technique.} 
In Theorem~\ref{thm:GLLRashomon}, we demonstrated that the proof technique of~\cite{gopi2022private} can be extended to obtain privacy for the Gibbs distribution used in Rashomon set sampling problem (i.e., when $\threshold>0$). However, unlike Lemma~\ref{lem:renyifinite}, the proof of Theorem~\ref{thm:GLLRashomon} requires an isoperimetric inequality based on LSI~\citep{ledoux1999concentration}, and bounding the LSI constant for the Rashomon set sampler. 
In addition, while both our work and~\citep{gopi2022private} shows uniform stability property for the Gibbs distribution, our proof is arguably simpler: Theorem~\ref{thm:rashomon-uniform} uses only well-known properties of gradient descent and avoids tools such as the Talagrand transportation inequality. Since the proof techniques of~\cite{gopi2022private} and ours are completely independent, we believe that both the techniques will find adoption in subsequent works on DP optimization. 

\paragraph{On run time} In the DP-SCO setting, the gradient oracle complexity of~\cite{gopi2022private} is better than our SGLD discretization of the Langevin diffusion. In particular, they achieve oracle complexity with logarithmic dependence on $p$ and the total variation distance error $\delta$, whereas we have a dependence $p / \delta^2$ on these two parameters. Furthermore, we require exact gradient oracles, whereas they only require an unbiased function oracle.

\section{Discrete Approximation of the Langevin Diffusion and SGLD}
\label{sec:sgld}
While exactly sampling from the Gibbs distribution may be computationally intractable, under some additional assumptions, a number of papers study polynomial-time algorithms for approximate sampling from the Gibbs distribution with various metric of approximation, such as Renyi divergence~\citep{vempala2019renyi, GaneshT20, erdogdu2021convergence, ChewiLSILD} and  $\infty$-divergence~\citep{BST14,mangoubi2021sampling}.
For sampling from the distribution $\exp(-\beta\calL)$, a popular approximate sampler is the Stochastic Gradient Langevin Dynamics (SGLD).
The SGLD approximates a finite-time solution to \eqref{eq:LD} via the discrete updates:
$ \theta_{t+1} = \theta_t - \eta \beta \nabla \calL(\theta_t; D) + \xi_t,$ where $\xi_t \sim \calN(0, 2 \eta I_p).$ 
This update can be seen as equivalent to \cref{eq:LD} if we use $\nabla \calL(\theta_{\eta \lfloor t/\eta\rfloor}; D)$ instead of $\nabla \calL(\theta_t; D)$, i.e., instead of continuously, update the gradient drift term in \cref{eq:LD} every $\eta$ time.

Many works study SGLD as an approximate sampler, and show that, for sufficiently small $\eta$ and large $T$, $\theta_T$ is an approximate sample from the stationary distribution $\exp(-\beta\calL)$ of \eqref{eq:LD}. For an appropriate definition of approximation, such as total variation distance, privacy of the stationary distribution then implies privacy of SGLD. SGLD also converges in polynomial time for stronger notions of convergence such as the R\'enyi divergence, but since divergences are ``one-directional'' bounds, whereas privacy requires ``bi-directional'' bounds, these results combined with the privacy of the stationary distribution do not necessarily ensure privacy of SGLD. One could also use the result of \cite{altschuler2022resolving}, which shows SGLD is an approximate sampler for the \textit{discrete chain's} stationary distribution. However, one would then need to show a bound on the bias due to the discretization that preserves privacy and utility. 

Instead, we appeal to the privacy of the noisy gradient steps taken since SGLD is just a reparameterization of DP-SGD. 
Using the composition theorem for R\'enyi divergences (Fact~\ref{fact:composition}) and translation from R\'enyi divergence bounds to $(\epsilon, \delta)$-DP (Fact~\ref{fact:renyitoapx}), we get the following.

\begin{lem}\label{lem:sgld-privacy}
If $$\ltwo{\nabla \calL(\theta; D) - \nabla \calL(\theta; D')} \leq \Delta \quad \text{and} \quad \beta \leq \frac{2 \epsilon}{\Delta \sqrt{T \eta \log(1/\delta)}},$$ then
outputting $\theta_T$ sampled from SGLD is $(\epsilon, \delta)$-DP.
\end{lem}
\vspace{-3mm}

This statement matches Theorem~\ref{lem:renyifinite} as $\eta \rightarrow 0$. For the utility guarantee of SGLD as an approximate sampler, we use results from \cite{ChewiLSILD}, though these require smoothness. The thresholded loss does not satisfy smoothness for $\threshold > 0$. In Appendix~\ref{app:kernel}, we show that using standard smoothing techniques, one can still get a discrete approximate sampler for the Gibbs distribution of a smoothed version of the thresholded loss, which implies the following sampler:

\begin{thm}\label{thm:sgld-sampler-tvd} 
Suppose $\ltwo{\nabla \calL(\theta; D) - \nabla \calL(\theta; D')} \leq \Delta$, and that each individual loss function $\ell$ is $m$-strongly convex and $M$-smooth. Let $0 \leq \lambda \leq \sqrt{\frac{\threshold}{Mp}}$ and $Q$ be the (unconstrained) Gibbs distribution of $\beta \widetilde{\calL}'$, where $$\widetilde{\calL}' := \mathbb{E}_{\xi \sim N(0, \lambda^2 \mathbb{I}_p)} \left[\min\{\calL(\theta + \xi; D), \min\limits_{\theta^* \in \calC}\calL(\theta^*; D) + \threshold\}\right].$$ 

Then for 
\[\beta = \widetilde{O}\left(\frac{\epsilon^2 m}{\max\{\Delta^2, M \threshold\}  \log^2(1/\delta)}\right), T= \widetilde{\Omega}\left(\frac{p (M^2 + \frac{M \threshold}{\lambda^2}) \max\left\{\Delta^4 , M^2 \threshold^2\right\}}{\epsilon^4 m^4 \delta^2} \right),\]
 there exists an algorithm using $T$ iterations  
that satisfies $(\epsilon, \delta)$-DP and returns a sample from a distribution whose TVD from $Q$ is $\delta$. Furthermore, the privacy guarantee holds even without convexity. In particular, if $\threshold = 0$, for $\lambda = 0$ the bounds are 
$\beta = \widetilde{O}\left(\frac{\epsilon^2 m}{\Delta^2 \log^2(1/\delta)}\right),
T= \widetilde{\Omega}\left(\frac{p M^2 \Delta^4}{\epsilon^4 m^4 \delta^2} \right).$
\end{thm}

In Theorem~\ref{thm:sgld-sampler-tvd}, we do not know if one can exactly compute the values of $\nabla \widetilde{\calL}'(\theta)$, but one can make approximate oracle calls to the gradients of $\widetilde{\calL}'$ via Monte Carlo sampling. That is, one samples $\xi_1, \ldots, \xi_k$, and then uses $\frac{1}{k} \sum_{i \in [k]} \nabla (\min\{\calL(\theta + \xi; D), \min\limits_{\theta^* \in \calC}\calL(\theta^*; D) + \threshold\})$ as an estimate of $\nabla \widetilde{\calL}'(\theta)$. It is easy to check that using Monte Carlo sampling instead of exact gradients does not affect our (worst-case) privacy analysis. Of course, in practice, the assumptions in Theorem~\ref{thm:sgld-sampler-tvd} may not hold anyway, and using the convolved loss function and using the unperturbed loss function will lead to similar outcomes for small $\lambda$, and also have the same privacy guarantees.
\section{Optimal DP-ERM/SCO Bounds from Rashomon Samplers}
\label{sec:DPERMSCO}

In this section we show that just using the Rashomon sampler with $\threshold = 0$ as a primitive is enough to derive near-optimal bounds for DP-ERM/SCO in all settings. Our results are summarized in Table~\ref{tb:resultsERMSCO}.

\begin{table}[t]
\begin{center}
\begin{tabular}{|c|c||c||c|}
\hline
 & Assumption & $\epsilon$-DP & $(\epsilon,\delta)$-DP\\
 \cline{2-4}
\hline
\hline
\multirow{2}{*}{DP-ERM} &
Convex &   $\frac{Lp}{\epsilon n}$    &    $\tilde{O}_\delta\left(\frac{L\sqrt{p}}{\epsilon n}\right)$         \\
\cline{2-4}
& SC &  {\color{blue} $\frac{L^2(p^2 + p \log n)}{m\epsilon^2 n^2}$ }       &   $\tilde{O}_\delta\left(\frac{L^2p}{m\epsilon^2 n^2}\right)$          \\
\cline{2-4}
\hline
\hline 
\multirow{2}{*}{DP-SCO} & Convex &   {\color{blue}$\frac{L}{\sqrt n}+\frac{Lp}{\epsilon n}$}   &    $\frac{L}{\sqrt n}+\tilde{O}_\delta\left(\frac{L\sqrt{p}}{\epsilon n}\right)$        \\
\cline{2-4}
& SC &  {\color{blue} $ \frac{L^2}{mn}+\frac{L^2p^2 \log n}{m \epsilon^2 n^2} $ }      &     
$\frac{L^2}{mn}+\tilde{O}_\delta\left(\frac{L^2p}{m\epsilon^2 n^2}\right) $            \\
\hline
\end{tabular}
\end{center}
\vspace{-5mm}
\caption{Summary of results that can be (re-)derived using the Rashomon sampler. The bounds marked in {\color{blue} blue} were not known even via different algorithms, and all other bounds are tight. Convex: Convex bounded Lipschitz losses, SC: Convex with $\nabla^2 \ell(\theta;\cdot)\succcurlyeq m \mathbb{I}$.} \label{tb:resultsERMSCO}
\end{table}

\subsection{Pure DP, Convex Losses}

For $L$-Lipschitz convex losses on $\calC$, the best possible excess empirical risk under $\epsilon$-DP is $O\left(\frac{L\ltwo{\calC}p}{\epsilon n}\right)$. This bound is achieved by the exponential mechanism as shown in \cite{BST14}, which is exactly what the Gibbs distribution is for $\threshold = 0$.

The best possible excess population risk under $\epsilon$-DP is $O\left(\frac{L\ltwo{\calC}p}{\epsilon n} + \frac{L \ltwo{\calC}}{\sqrt{n}}\right)$. We can achieve this via Theorem~\ref{thm:rashomon-population} by setting $\threshold = 0$ and adding the regularizer $\frac{L}{2 \ltwo{\calC} \sqrt{n}} \ltwo{\theta - \theta_0}^2$, where $\theta_0$ is an arbitrary point in $\calC$, to the loss function to give the algorithm uniform stability. The excess population risk of the Gibbs distribution with respect to the regularized loss is $O(\frac{L \ltwo{\calC} p}{\epsilon n} + \frac{L \ltwo{\calC}}{\sqrt{n}}))$ by Theorem~\ref{thm:rashomon-population}, and the regularized and unregularized loss differ by at most $O(\frac{L \ltwo{\calC}}{\sqrt{n}})$ everywhere in $\calC$.  Putting it all together, we get the following:

\begin{thm}
\label{thm:convexpureDPSCO}
For convex, $L$-Lipschitz losses over $\calC$, there exists an $\epsilon$-DP algorithm 
with excess population risk $O\left(\frac{L\ltwo{\calC}p}{\epsilon n} + \frac{L \ltwo{\calC}}{\sqrt{n}}\right)$.\label{thm:expConvSCO}
\end{thm}

\subsection{Pure DP and Strongly Convex Losses}

For $L$-Lipschitz, $m$-strongly convex losses on $\calC$, the best possible excess empirical risk under $\epsilon$-DP is $O\left(\frac{L^2 p^2}{\epsilon^2 n^2 m}\right)$. Unfortunately, the guarantee given by Theorem~\ref{thm:expConv-rashomon-util} is worse by a quadratic factor. In \cite{BST14}, the optimal excess empirical risk is achieved (up to log factors) by first choosing a smaller ball using the Laplace mechanism, and then running the Gibbs sampler on this smaller ball. We show the best-known bound can be achieved using the Gibbs sampler only as a primitive, which does not allow us to use the Laplace mechanism. We propose the iterated exponential mechanism, given as Algorithm~\ref{alg:iem}, with the following guarantee:

\begin{algorithm}
\caption{Iterated Exponential Mechanism}
\begin{algorithmic}[1]
    \Require {Loss function $\calL$, constraint set $\calC_0$, Lipschitz constant $L$, strong convexity parameter $m$, number of iterations $k$, privacy parameter sequence $\{\epsilon_i\}_{i=1}^k$, $\mathsf{flag}$ and data set $\dataset$ of $n$ samples.}

    \State {\textbf{for} $i = 1$ to $k$ \textbf{do}} 
        \State \hspace{5mm} {Sample $\theta_i$ from $\calC_{i-1}$ with probability proportional to $\exp( - \frac{\epsilon_i n}{2 L \ltwo{\calC_{i-1}}} \calL(\theta; \dataset))$.}
        \State \hspace{5mm}\textbf{If} $\mathsf{flag}=1$ \textbf{then}
        
        \State \hspace{10mm}{$\calC_i \leftarrow \set{\theta \in \calC_{i-1} : \ltwo{\theta - \theta_i} \leq \sqrt{\frac{c L (p + 3 \log n) \ltwo{\calC_{i-1}}}{m \epsilon_i n}} } $.}
        \State \hspace{5mm} \textbf{else}
        \State \hspace{10mm} 
        {$\calC_i \leftarrow \set{\theta \in \calC_{i-1} : \ltwo{\theta - \theta_i} \leq \sqrt{\frac{cL(p + 3 \log n) \ltwo{\calC_{i-1}}}{m\epsilon_i n}} + \frac{cL\sqrt{\log n}}{m \sqrt{n}} } $.}
    \State \textbf{end}
\State{\Return{$\theta_k$}}
\end{algorithmic}
\label{alg:iem}
\end{algorithm}

\begin{thm}
\label{thm:expStrongConv}
Assume each of the individual loss function in $\calL(\theta;D)$ is $L$-Lipschitz within the constraint set $\calC_0$. For any $\epsilon$, if we instantiate  Algorithm~\ref{alg:iem} with $k = 1 + \lceil \log \log (\frac{ \epsilon n}{(p + \log n)}) \rceil$ and $\epsilon_i = \epsilon /2^{k-i+1}$, then Algorithm~\ref{alg:iem} is $\epsilon$-differentially private. 

Additionally, if the loss function $\calL(\theta;D)$ is $m$-strongly convex and the constraint set $\calC_0$ is convex and $\mathsf{flag=1}$, then over the randomness of the algorithm, the output $\theta_k$ of Algorithm~\ref{alg:iem} has excess empirical risk:
\[O\left(\frac{L^2 (p^2 + p \log n)}{\epsilon^2 n^2 m}\right).\]
\end{thm}

The theorem follows by solving a recurrence for $\ltwo{\calC_i}$ to bound the diameter of the final set $\calC_{k-1}$. Then we show that the minimizer over $\calC_0$ is also in $\calC_{k-1}$ with high probability. Therefore, the analysis of the exponential mechanism on $\calC_{k-1}$ gives the theorem. We note that in addition to only using the Gibbs sampler as a primitive, we improve the $p^2 \log n$ in \cite{BST14} result to $p^2 + p \log n$.
To bound the excess loss, we first need the following lemma, which shows that with high probability we choose a series of $\calC_i$ that all contain the optimal $\theta$ for $\calC_0$. 
It follows from a tail bound on the excess loss of the exponential mechanism, and using $m$-strong convexity to translate this into a distance bound.

\begin{lem}\label{lem:optinball}
Let $\calL(\cdot;D)$ be an $m$-strongly convex function. Suppose we sample $\theta$ from the convex constraint set $\calC$ with probability proportional to $\exp\left(-\frac{\epsilon n}{2 L \ltwo{\calC}} \calL(\theta; \dataset)\right)$. Let $\theta^* = \argmin_{\theta \in \calC} \calL(\theta; \dataset)$. Then for any $t \geq 0$ and for some sufficiently large constant $c$ we have  

\begin{align*}\Pr \sparen{ \ltwo{\theta - \theta^*} \leq \sqrt{\frac{cL(p+t)\ltwo{\calC}}{m \epsilon n}}} \geq 1-2^{-t}.\end{align*}
\end{lem}

\begin{proof}
By e.g. the proof of \cite[Theorem III.2]{BST14}, we know that for some sufficiently large constant $c$:

\begin{align}
\Pr\left[\calL(\theta; \dataset) - \calL(\theta^*; \dataset) \leq \frac{cL\ltwo{\calC}}{2\epsilon n}(p + t)\right] \geq 1 - 2^{-t}.
\label{eq:bst14_exp}
\end{align}

We now show that the claim holds conditioned on this event. By optimality of $\theta^*$ and convexity of $\calC$, we know 
\begin{align}
\langle \grad \calL(\theta^*; \dataset), \theta - \theta^* \rangle \geq 0.
\label{eq:optimality_exp}
\end{align}

So, by $m$-strong convexity, we have

\begin{align*}
\frac{cL\ltwo{\calC}}{2\epsilon n} (p+t) &\underset{(\ref{eq:bst14_exp})}{\geq} \calL(\theta; \dataset) - \calL(\theta^*; \dataset) \\
&\geq \langle \grad \calL(\theta^*; \dataset), \theta - \theta^* \rangle + \frac{m}{2}\ltwo{\theta - \theta^*}^2 \\ &\underset{(\ref{eq:optimality_exp})}{\geq} \frac{m}{2}\ltwo{\theta - \theta^*}^2.\end{align*}

Rearranging gives the claim in Lemma~\ref{lem:optinball}.
\end{proof} 

Given Lemma~\ref{lem:optinball}, we can now prove Theorem~\ref{thm:expStrongConv}.

\begin{proof}[Proof of Theorem~\ref{thm:expStrongConv}]
The privacy guarantee is immediate from the privacy guarantee of the exponential mechanism, composition, and the fact that for this choice of $\epsilon_i, k$, we have $\sum_{i=1}^k \epsilon_i < \epsilon$.

Setting $t = 3 \log n$ in Lemma~\ref{lem:optinball}, in iteration $i$, letting $\theta^*_i =  \argmin_{\theta \in \calC_{i-1}} \calL(\theta; \dataset)$, we have that with probability $1-2^{-t} = 1 - \frac{1}{n^3}$, $\theta^*_i \in \calC_i$, and thus $\theta^*_i = \theta^*_{i+1}$. Then by a union bound, we have that with probability $1 - \frac{k}{n^3} \geq 1 - \frac{\log \log (\epsilon n)}{n^3}$, $\theta^*_1 \in \calC_{k-1}$ (equivalently, $\theta^*_1 = \theta^*_2 = \ldots = \theta^*_k$). When this event fails to happen, our excess loss is at most $L\ltwo{\calC_0}$, and in turn the contribution of this event failing to hold to the expected excess loss is $O\left(\frac{L \ltwo{\calC_0} \log \log (\epsilon n)}{n^3}\right)$, which is asymptotically less than our desired excess loss bound. So it suffices to provide the desired expected excess loss bound conditioned on this event. By the analysis of the exponential mechanism, conditioned on this event, we have that 

\begin{align}
 \mathds{E}_{\theta_k}\left[\calL(\theta_k; \dataset)\right] - \calL(\theta_1^*; \dataset) = O\left(\frac{Lp \ltwo{\calC_{k-1}}}{\epsilon_k n} \right) = O\left(\frac{Lp \ltwo{\calC_{k-1}}}{\epsilon n}\right).
 \label{eq:purestronglyconvloss}
\end{align}

Note that $\calL(\theta_1^*; \dataset) = \min_{\theta \in \calC_0} \calL(\theta; \dataset)$ by definition, so it now suffices to bound $\ltwo{\calC_{k-1}}$ by $O\left(\frac{L(p + \log n)}{m \epsilon n}\right)$. To do this, we have the recurrence relation:

\begin{align*}\ltwo{\calC_i} \leq 2 \sqrt{\frac{c L (p + 3 \log n) \ltwo{\calC_{i-1}}}{m \epsilon_i n}}.\end{align*}

Solving the recurrence relation for $\calC_{k-1}$, we get:

\begin{align}
\begin{split}
    \ltwo{\calC_{k-1}} 
    &\leq \left(\frac{4 c L (p + 3 \log n) }{m  n}\right)^{1-2^{-(k-1)}} \cdot (\ltwo{\calC_0})^{2^{-(k-1)}} \cdot \prod_{i=1}^{k-1} \epsilon_i^{-2^{-(k-i)}} \\
    & =\left(\frac{4 c L (p + 3 \log n) }{m  \epsilon n}\right)^{1-2^{-(k-1)}} \cdot (\ltwo{\calC_0})^{2^{-(k-1)}} \cdot \prod_{i=1}^{k-1} (2^{(k-i+1)})^{2^{-(k-i)}}.
\end{split}
\label{eq:c_k_1}
\end{align}

We claim the following: 
\begin{align}
    \ltwo{\calC_0} \leq \frac{2L}m. 
    \label{eq:calC_0}
\end{align}

Let $\theta_{\mathsf{global}}$ be the minimizer of $\calL(\theta; \dataset)$ over all of $\mathds{R}^p$. By triangle inequality, there exists a point $\theta$ in $\calC_0$ which is at distance at least $\ltwo{\calC_0}/2$ far from $\theta_{\mathsf{global}}$. By $m$-strong convexity, this implies that the gradient at $\theta$ has $\ell_2$-norm at least $m \ltwo{\calC_0}/2$. Now, by Lipschitzness over $\calC_0$, we know that the gradient at $\theta$ has $\ell_2$-norm at most $L$. This gives us \cref{eq:calC_0}. 

Using \cref{eq:calC_0}, we can simplify \cref{eq:c_k_1} to

\begin{align*}\ltwo{\calC_{k-1}} \leq \frac{2L}{m} \cdot \left(\frac{2c (p + 3 \log n) }{ \epsilon n}\right)^{1-2^{-(k-1)}} \cdot \prod_{i=1}^{k-1} (2^{(k-i+1)})^{2^{-(k-i)}}.\end{align*}

We have:

\begin{align*}\log_2 \left(\prod_{i=1}^{k-1} (2^{(k-i+1)})^{2^{-(k-i)}}\right) = \sum_{i=1}^{k-1} (k-i+1)2^{-(k-i)} \leq \sum_{j=1}^{\infty} (j+1)2^{-j} = 3.\end{align*}

In other words, $\prod_{i=1}^{k-1} (2^{(k-i+1)})^{2^{-(k-i)}}$ is at most $8$, regardless of the value of $k$. Now, using the fact that $m^{1/\log m} = O(1)$ is a constant, our final upper bound on $\ltwo{\calC_{k-1}}$ is:

\begin{align*}\ltwo{\calC_{k-1}} = O\left( \frac{L}{m} \cdot \left(\frac{(p + \log n) }{ \epsilon n}\right)^{1-2^{-(k-1)}} \right) = O\left(\frac{L(p + \log n) }{m \epsilon n}\right).\end{align*}
Plugging in \cref{eq:purestronglyconvloss} gives us \Cref{thm:expStrongConv}.
\end{proof}

The best possible population risk bound is $O(\frac{L^2 p^2}{\epsilon^2 n^2 m} + \frac{L^2}{m n})$.
In order for Algorithm~\ref{alg:iem} to achieve this bound (up to log factors) we make a slight modification: we choose the radius of each ball defined by the algorithm such that the population minimizer, rather than the empirical minimizer, is in $\calC_{k-1}$ with high probability. Then, we can apply uniform stability of the Gibbs sampler on strongly convex losses to the exponential mechanism run on $\calC_{k-1}$ to get the following DP-SCO bound:

\begin{thm}
\label{thm:stronglyconvexpureDPSCO}
Let $\theta_k$ be the output of Algorithm~\ref{alg:iem} 
when $\mathsf{flag}=0$. 
Then $\theta_k$ has excess population risk
\[O\left(\frac{L^2 p^2 \log n}{m \epsilon^2 n^2} + \frac{L^2}{m n}\right).\]
\label{thm:strongERMpureDP}
\end{thm}
We first show that the empirical minimizer is close to the population minimizer with high probability:

\begin{lem}\label{lem:empnearpop}
Let $\ell$ be a $m$-strongly convex function and $\calC$ be a convex set such that for any $d, \theta$, $$\ltwo{\grad \ell(\theta; d) - \mathbb{E}_{d \sim \dist}\left[\grad \ell(\theta; d)\right]} \leq \Delta,$$ and let $\theta^*:= \argmin_{\theta \in \calC} \mathbb{E}_{d \sim \dist}\left[\ell(\theta; d)\right]$ and $\theta^{\mathsf{emp}} := \argmin_{\theta \in \calC} \ell(\theta; \dataset)$. 
Then for $\dataset \sim \dist^n$, with probability $1-\gamma$, we have:

\[\ltwo{\theta^{\mathsf{emp}} - \theta^*} = O\left(\frac{\Delta}{m} \sqrt{\log(1/\gamma) \over n}\right)\]
\end{lem}

\begin{proof}
Consider a function $\tilde{\ell}$ which has gradient $\grad \tilde{\ell}(\theta) = \grad \ell(\Pi_\calC(\theta)) + m (\theta - \Pi_\calC(\theta))$. For any $\dataset$, the empirical minimizer of $\tilde{\ell}$ over $\mathbb{R}^p$ is equal to 
$$\tilde{\theta}^{\mathsf{emp}} := \theta^{\mathsf{emp}} - \frac{1}{m} \cdot \grad \ell(\theta^{\mathsf{emp}}; \dataset),$$ 
and the population minimizer of $\mathbb{E}_{d \sim \dist}\left[\tilde{\ell}(\theta; d)\right]$ is 
$$\tilde{\theta}^* := \theta^* - \frac{1}{m} \cdot \mathbb{E}_{d \sim \dist}\left[\grad \ell(\theta^*; d)\right].$$ 

By optimality of $\theta^{\mathsf{emp}}, \theta^*$ and convexity of $\calC$, 
$$\langle\grad \ell(\theta^{\mathsf{emp}}; \dataset), \theta^* - \theta^{\mathsf{emp}} \rangle \geq 0 \quad \text{and} \quad \langle \mathbb{E}_{d \sim \dist}\left[\grad \ell(\theta^*; \dataset)\right], \theta^{\mathsf{emp}}-\theta^* \rangle \geq 0,$$ 

This implies that $\ltwo{\tilde{\theta}^{\mathsf{emp}} - \tilde{\theta}^*} \geq \ltwo{\theta^{\mathsf{emp}} - \theta^*}$. In addition, since projection to convex sets is a non-expansive map, $\tilde{\ell}$ is $m$-strongly convex if $\ell$ is, and for any $d, \theta$ we have 
$$\ltwo{\grad \ell(\theta; d) - \mathbb{E}_{d \sim \dist}\left[\grad \ell(\theta; d)\right]} = \ltwo{\grad \tilde{\ell}(\theta; d) - \mathbb{E}_{d \sim \dist}\left[\grad \tilde{\ell}(\theta; d)\right]} \leq \Delta.$$ 

This holds for any $\dataset$. Therefore, if we prove the lemma for $\tilde{\ell}$ and $\mathbb{R}^p$, then this would imply that the lemma holds for $\ell$ and $\calC$. So it suffices to show the lemma for $\calC = \mathbb{R}^p$.

If $\calC = \mathbb{R}^p$ then $\mathbb{E}_{d \sim \dist}\left[\grad \ell(\theta; d)\right] = \boldzero$.
Now, by the assumptions in the lemma and a vector Azuma inequality \citep{hayes03vectorazuma}, we have $\ltwo{\grad \ell(\theta^*; \dataset)} =  O(\frac{\Delta\sqrt{\log(1/\gamma)}}{\sqrt{n}})$ with probability $1 - \gamma$ over $\dataset$. Furthermore, we know $\grad \ell(\theta^{\mathsf{emp}}; \dataset) = \boldzero$ by strong convexity and since $\calC = \mathbb{R}^p$. Then by strong convexity, we have 
$$\ltwo{\theta^* - \theta^{\mathsf{emp}}} \leq \frac{\ltwo{\grad \ell(\theta^*; \dataset) - \grad \ell(\theta^{\mathsf{emp}}; \dataset)}}{m} = \frac{\ltwo{\grad \ell(\theta^*; \dataset)}}{m} = O\left(\frac{\Delta}{m} \sqrt{\log(1/\gamma) \over n}\right)$$ 
with probability $1-\gamma$ as desired in the statement of Lemma~\ref{lem:empnearpop}.
\end{proof}

Given Lemma~\ref{lem:empnearpop}, if we want to ensure the \textit{population} minimizer rather than empirical minimizer remains in the sets we choose in Algorithm~\ref{alg:iem}, we just need to choose a slightly larger ball. From this modification and uniform stability, we get our DP-SCO bound:

\begin{proof}[Proof of Theorem~\ref{thm:stronglyconvexpureDPSCO}]
Note that by $L$-Lipschitzness of $\ell$ in $\calC$, we have 

\[\ltwo{\grad \ell(\theta; d) - \mathbb{E}_{d \sim \dist}\left[\grad \ell(\theta; d)\right]} \leq 2L.\]

By Lemma~\ref{lem:empnearpop}, Lemma~\ref{lem:optinball},  and a triangle inequality, we have that the population minimizer of $\ell$ in $\calC_i$ is in $\calC_{i+1}$ for each $i$ with probability $1-2/n^3$. Then by a union bound, we have that the population minimizer is in $\calC_{k-1}$. When this event fails to hold, our excess population risk is $O(L\ltwo{\calC_0})$ and so the contribution of this event to the expected excess loss is $O\left(\frac{L\ltwo{\calC_0} \log \log(\epsilon n)}{n^3}\right) = O\left(\frac{L^2 \log \log(\epsilon n)}{mn^3}\right)$, which is asymptotically less than our desired bound. So it suffices to provide the desired expected excess loss bound conditioned on this event. We can bound the radius of $\calC_{k-1}$ similarly to the proof of Theorem~\ref{thm:expStrongConv}, by noting that:

\[\ltwo{\calC_i} \leq 2 \cdot \max\left\{\sqrt{\frac{cL(p + 3 \log n) \ltwo{\calC_{i-1}}}{m\epsilon_i n}}, \frac{cL\sqrt{\log n}}{m \sqrt{n}}\right\}\]

Then, rolling out the recursion, we have similarly to the proof of Theorem~\ref{thm:expStrongConv}:

\[\ltwo{\calC_{k-1}} = O\left(\frac{L(p + \log n)}{m \epsilon n} + \frac{L\sqrt{\log n}}{m \sqrt{n}}\right).\]

Now, combining Theorem~\ref{thm:expMech} and the uniform stability bound of Theorem~\ref{thm:rashomon-uniform} (for $\threshold = 0$), we get that the expected excess population risk of $\theta_k$ compared to the population minimizer over $\calC_{k-1}$ is:

\begin{align*}
    O\left(\frac{Lp\ltwo{\calC_{k-1}}}{\epsilon n} + \frac{L^2}{mn}\right) &= O\left(\frac{L^2}{mn} \cdot \left(\frac{(p^2 + \log^2 n)}{\epsilon^2 n} + \frac{ p \sqrt{\log n}}{\epsilon\sqrt{n}}+ 1 \right)\right) \\
    &= O\left(\frac{L^2 (p^2 \log n + \log^2 n)}{m \epsilon^2 n^2} + \frac{L^2}{mn}\right).
\end{align*}

In the final equality, we use the fact that $\frac{p}{\epsilon} \sqrt{\log n \over n} \leq \max\left\{\frac{p^2 \log n}{\epsilon^2 n}, 1\right\}$.
We conclude by noting that conditioned on the event the population minimizer over $\calC_0$ is contained in $\calC_{k-1}$, $\theta_k$ has this same excess population risk bound compared to the population minimizer over $\calC_0$, completing the proof of Theorem~\ref{thm:stronglyconvexpureDPSCO}.
\end{proof}

\subsection{Approximate DP and Strongly Convex Losses}

The results in Theorem~\ref{thm:expConv-rashomon-util}, \ref{thm:GLLRashomon} and \ref{thm:rashomon-population} with $\threshold = 0$ combined immediately give that the Gibbs sampler achieves the optimal bounds of $O(\frac{L^2 p \log(1/\delta)}{\epsilon^2 n^2 m})$ and $O(\frac{L^2 p \log(1/\delta)}{\epsilon^2 n^2 m} + \frac{L^2}{m n})$ for excess empirical and population risk respectively in this setting. We note that one could also use Theorem~\ref{thm:expMech-rashomon-approx} instead of Theorem~\ref{thm:GLLRashomon} and obtain bounds that are within logarithmic factors of the optimal bounds, with a finite-time object.

\subsection{Approximate DP and Convex Losses}

Similarly to pure-DP, we can adapt our results in the strongly convex setting to the convex setting by adding a regularizer. For a near-optimal empirical guarantee, we use the regularizer $\frac{L \sqrt{p  \log(1/\delta)}}{2 \ltwo{\calC} \epsilon n} \ltwo{\theta - \theta_0}^2$. The regularized and unregularized losses differ by at most $O\left(\frac{L \ltwo{\calC} \sqrt{p  \log(1/\delta)}}{\epsilon n}\right)$ everywhere in $\calC$, and the empirical excess loss bound we get by plugging in the strong convexity parameter $\frac{L \sqrt{p  \log(1/\delta)}}{ \ltwo{\calC} \epsilon n}$ into the bound for strongly convex losses is $O\left(\frac{L \ltwo{\calC} \sqrt{p  \log(1/\delta)}}{\epsilon n}\right)$, giving nearly the optimal bound of $O\left(\frac{L \ltwo{\calC} \sqrt{p  \log(1/\delta)}}{\epsilon n}\right)$.

For population risk, we use the regularizer $\frac{L}{2 \ltwo{\calC}} \cdot \max\left\{\frac{\sqrt{p \log(1/\delta)}}{\epsilon n}, \frac{1}{\sqrt{n}}\right\}  \ltwo{\theta - \theta_0}^2$. By a similar argument, this gives the optimal bound of $O\left(\frac{L \ltwo{\calC} \sqrt{p  \log(1/\delta)}}{\epsilon n} + \frac{L \ltwo{\calC}}{\sqrt{n}}\right)$.
\section{Discussion and Future Directions}
\label{sec:discussion}

In this work we demonstrated the power of Langevin diffusion (LD) by simultaneously obtaining tight guarantees for DP-ERM, DP-SCO, and obtaining the the first private uniform sampling algorithms from Rashomon sets. Furthermore, we demonstrated that, via SGLD, it is possible to maintain the same privacy/utility trade-offs whiling allowing the algorithm to implemented on a finite precision machine. We believe the idea of using a LD to analyze the privacy of the mechanism that samples from the Gibbs distribution, and using a LD to analyze the uniform stability of this mechanism has wider applicability in the DP literature. We leave it for future exploration. Furthermore, the gradient complexity of our SGLD algorithm is inferior to that of~\cite{gopi2022private}. It is an important open question if it is at all possible to close this gap while not relying on the convexity of the loss function as in our case. This would have a significant real world implications where we aim to ensure privacy and also train learning models that are inherently non-convex. Finally, we believe that exploring other applications of Rashomon sets would would facilitate wider adoption of machine learning models that are differentially private.

\section*{Acknowledgements}
We would like to thank Walid Krichene, Dvijotham Krishnamurthy, Ryan McKenna, Sewoong Oh, Adam Smith, and Thomas Steinke for their helpful comments.

\newpage
\bibliographystyle{alpha}
\bibliography{reference}
\newpage

\appendix

\section{Notation and Preliminaries}
\label{sec:back}
\anoteinline{Need to remove the background that is not needed, and possibly add more that is needed.}
\begin{table}[t]
\begin{center}
\begin{tabular}{|c||c|c||c|c|}
\hline
 Notation &   \\
\hline
$D = \set{d_1,\cdots, d_n}$ & data set \\ \hline
$\calD$ & data distribution \\ \hline
$\tau$ & domain set of data  \\ \hline
$\calC \subset \mathbb R^p$ & convex set/parameter space \\ \hline
$\ell$ & loss function  \\ \hline
$\calL$ & empirical loss function  \\ \hline
$\riskerm$ & excess empirical risk \\ \hline
$\riskpop$ & excess population risk \\ \hline
$m$ & strong convexity parameter \\ \hline
$M$ & smoothness parameter \\ \hline
$L$ & Lipschitz constant \\ \hline
$\privT$ & private model output \\ \hline
$\theta^*$ & optimal model \\ \hline
$\theta^{\mathsf{emp}}$ & model for ERM \\ \hline
$\beta$ & inverse temperature \\ \hline
$W_t$ & standard Brownian motion \\ \hline
$R_\alpha(\cdot, \cdot)$ & Renyi divergence of order $\alpha$ \\ \hline
$T$ & time \\ \hline 
$A \succeq 0$ & $A$ is positive semidefinite \\ \hline $A \succeq B $ & $A-B$ is positive semidefinite \\ \hline 
$\mathbb I_p$ & $p \times p$ identity matrix \\ \hline 
$(\epsilon,\delta)$ & Privacy parameters \\ \hline
$\gamma$ & Empirical loss slack \\ \hline 
$\gamma'$ & Population loss slack \\ \hline 
$\threshold$ & Threshold of Rashomon set \\ \hline 
$\lambda$ & Sampling error wrt TVD \\ \hline
\end{tabular}
\end{center}
\caption{Notation Table}
\label{tb:notation}
\end{table}

In this section, we give a brief exposition of the concepts and results used in the rest of the paper. In Table~\ref{tb:notation} we provide a summary of the notation used in the paper.

\paragraph{R\' enyi divergence and differential privacy.}
R\'enyi divergence is the generalization of KL divergence to higher order and satisfies many useful properties~\citep{vanErvenH14}. More formally,  
\begin{defn}[R\'enyi Divergence]
For $0 < \alpha < \infty$, $\alpha \neq 1$ and distributions $P, Q$, such that $\supp(P) = \supp(Q)$ the {\em $\alpha$-R\'enyi divergence} between $P$ and $Q$ is

$$ R_\alpha(P , Q) = \frac{1}{\alpha - 1} \ln \int\limits_{\supp(Q)} \frac{P(x)^\alpha}{Q(x)^{\alpha - 1}} dx = \frac{1}{\alpha - 1} \ln \mathbb{E}_{x \sim Q}\left[ \frac{P(x)^\alpha}{Q(x)^\alpha}\right].$$

The $\alpha$-R\'enyi divergence for $\alpha = 1$ (resp. $\infty$) is defined by taking the limit of $R_\alpha(P,Q)$ as $\alpha$ approaches $1$ (resp. $\infty$) and equals the KL divergence (resp. max divergence).
\end{defn}

We next define differential privacy, our choice of the notion of data privacy. Central to the notion of differential privacy is the definition of {\em adjacent} or {\em neighboring} datasets. Two datasets $D$ and $D'$ are called adjacent if they differ in exactly one data point.
\begin{defn}
[Approximate Differential privacy~\citep{DMNS}]
A randomized mechanism $\calM: \calD^n \to \calR$ is said to have {\em $(\epsilon,\delta)$-differential privacy} , or $(\epsilon,\delta)$-DP for short, if for any adjacent $D,D' \in \calD^n$ and measurable subset $S \subset \calR$, it holds that 
\[
\mathsf{Pr}[\calM(D) \in S] \leq e^\epsilon \mathsf{Pr}[\calM(D) \in S] + \delta .
\]
When $\delta=0$, it is known as pure differential privacy, and we denote it by $\epsilon$-DP.
\end{defn}

\begin{defn}
[Renyi Differential privacy~\citep{mironov2017renyi}]
A randomized mechanism $\calM: \calD^n \to \calR$ is said to have {\em $(\alpha, \epsilon)$-R\'enyi differential privacy}, or $(\alpha,\epsilon)$-RDP for short, if for any adjacent $D,D' \in \calD^n$ it holds that
\[
R_\alpha(\calM(D), \calM(D') ) \leq \epsilon.
\]
\end{defn}

It is easy to see that $\epsilon$-DP is merely $(\infty,\epsilon)$-RDP. Similarly, the following fact relates $(\epsilon,\delta)$-DP to $(\alpha,\epsilon)$-RDP:
\begin{fact}[Proposition 3 in~\cite{mironov2017renyi}]
\label{fact:renyitoapx}
If $\mathcal{M}$ satisfies $(\alpha, \epsilon)$-RDP, then $\mathcal{M}$ is $(\epsilon + \frac{\log 1/\delta}{\alpha - 1}, \delta)$-differentially private for any $0 < \delta < 1$.
\end{fact}

R\'enyi divergences satisfy a number of other useful properties, which we list here.

\begin{fact}[Monotonicity {\citep[Theorem 3]{vanErvenH14}}]
\label{fact:monotonicity}
For any distributions $P, Q$ and $0 \leq \alpha_1 \leq \alpha_2$ we have $R_{\alpha_1}(P , Q)$ $\leq R_{\alpha_2}(P,Q)$. 
\end{fact}

\begin{fact}[Post-Processing {\citep[Theorem 9]{vanErvenH14}}]\label{fact:postprocessing}
For any sample spaces $\mathcal{X}, \mathcal{Y}$, distributions $P, Q$ over $\mathcal{X}$, and any function $f:\mathcal{X} \rightarrow \mathcal{Y}$ we have $R_\alpha(f(P) ,f(Q)) \leq R_\alpha(P , Q)$.
\end{fact}

\begin{lem}
[Gaussian dichotomy {\citep[Example 3]{vanErvenH14}}]\label{fact:gaussiandivergence}
Let $\calP = \calP_1 \times \calP_2 \times \cdots $ and $\calQ = \calQ_1 \times \calQ_2 \times \cdots $, where $\calP_i$ and $\calQ_i$ are unit variance Gaussian distributions with mean $\mu_i$ and $\nu_i$, respectively. Then 
\[
R_\alpha(\calP_i , \calQ_i) = \frac{\alpha}{2} (\mu_i-\nu_i)^2,
\]
and by additivity for $\alpha >0$,
\[
R_\alpha(\calP, \calQ) = \frac{\alpha}{2} \sum_{i=1}^\infty (\mu_i-\nu_i)^2.
\]
As a corollary, we have:
\[R_\alpha(N(0, \sigma^2 \mathbb{I}_p) , N(\bfx, \sigma^2 \mathbb{I}_p)) \leq \frac{\alpha \ltwo{\bfx}^2}{2\sigma^2}.\]
\end{lem}

\begin{fact}[Adaptive Composition Theorem {\citep[Proposition 1]{mironov2017renyi}}]\label{fact:composition}
Let $\mathcal{X}_0,$ $\mathcal{X}_1, \ldots, \mathcal{X}_k$ be arbitrary sample spaces. For each $i \in [k]$, let $f_i, f_i':\Delta(\mathcal{X}_{i-1}) \rightarrow \Delta(\mathcal{X}_i)$ be maps from distributions over $\mathcal{X}_{i-1}$ to distributions over $\mathcal{X}_i$ such that for any distribution $X_{i-1}$ over $\mathcal{X}_{i-1}$, $R_\alpha(f_i(X_{i-1}) , f_i'(X_{i-1})) \leq \epsilon_i$. Then, for $F, F':\Delta(\mathcal{X}_0) \rightarrow \Delta(\mathcal{X}_k)$ defined as $F(\cdot) = f_k(f_{k-1}( \ldots f_1(\cdot) \ldots )$ and $F'(\cdot) = f'_k(f'_{k-1}( \ldots f'_1(\cdot) \ldots )$ we have $R_\alpha(F(X_0) , F'(X_0)) \leq \sum_{i=1}^k \epsilon_i$ for any $X_0 \in \Delta(\mathcal{X}_0)$.
\end{fact}

\begin{fact}[Weak Triangle Inequality {\citep[Proposition 11]{mironov2017renyi}}]\label{fact:renyitriangle}
For any $\alpha > 1$, $q > 1$ and distributions $\calP_1, \calP_2, \calP_3$ with the same support:

$$R_\alpha(\calP_1, \calP_3) \leq \frac{\alpha - 1/q}{\alpha - 1}R_{q\alpha}(\calP_1, \calP_2) + R_{\frac{q\alpha - 1}{q - 1}}(\calP_2, \calP_3).$$
\end{fact}

We discuss two differentially private mechanisms for optimization in this paper. The first one is the {\em exponential mechanism.}~\citep{mcsherry2007mechanism}. Given some arbitrary domain $\mathfrak{D}$ and range $\mathfrak{R}$, the exponential mechanism is defined with respect to some loss function, $\ell: \mathfrak{D} \times \mathfrak{R} \to \mathbb R$.   
\begin{defn}
[Exponential mechanism~\citep{mcsherry2007mechanism}] 
Given a privacy parameter $\epsilon$, the range $\mathfrak R$ and a loss function $\ell : \mathfrak D \times \mathfrak R \to \mathbb R$, the {\em exponential mechanism} samples a single element from $\mathfrak R$ based on the probability distribution
\[
\pi_D(r) = \frac{e^{-\epsilon \ell(D,r)/2\Delta_\ell }}{\sum_{r \in \mathfrak R} e^{-\epsilon \ell(D,r)/2\Delta_\ell }}
\]
where $\Delta_\ell$ is the sensitivity of $u$, defined as $\Delta_\ell := \max_{D \sim D', \atop r \in \mathfrak R}|u(D,r) - u(D',r)|$. 
If $\mathfrak{R}$ is continuous, we instead sample from the distribution with pdf:
\[
p_D(r) = \frac{e^{-\epsilon \ell(D,r)/2\Delta_\ell }}{\int_{r \in \mathfrak R} e^{-\epsilon \ell(D,r)/2\Delta_\ell } dr}.
\] 
\end{defn}

\begin{algorithm}[ht]
\caption{Exponential mechanism}
\textbf{Input:}  Loss function $\calL$, constraint set $\calC$, Lipschitz constant $L$, number of iterations $k$, privacy parameter $\epsilon$, data set $\dataset$ of $n$-samples.
\begin{algorithmic}[1]
\State{Sample and {\bf output} a point $\privT$ from the constraint set $\calC$ w.p.  $\propto\exp\left(-\frac{\epsilon n}{2L\ltwo{\calC}}\cdot\calL(\theta;D)\right)$.\label{step:exp1}}
\end{algorithmic}
\label{alg:expmech}
\end{algorithm}

\begin{thm}
\label{thm:expConv}
Assume each of the individual loss function in $\calL(\theta;D)$ is $L$-Lipschitz within the constraint set $\calC$, individual loss function $\ell(\theta;\cdot)$ is convex, and the constraint set $\calC$ is convex. Then, Algorithm~\ref{alg:expmech} is $\epsilon$-differentially private. Furthermore, for $\privT$ as specified in Algorithm~\ref{alg:expmech}, over the randomness of the algorithm, 
\[\mathbb{E}_{\privT}[\riskerm(\privT)] = O\left(\frac{Lp\cdot\ltwo{\calC}}{\epsilon n}\right).\]
\label{thm:expMech}
\end{thm}

\mypar{Equivalence of Algorithm~\ref{alg:expmech} and Langevin diffusion} The following lemma, which is implied by, e.g. \cite[Theorem 4.1]{TanakaPLD}, shows that one can implement Algorithm 1 using only solutions to \cref{eq:PLD}; note that this does not necessarily mean solutions to \cref{eq:PLD} are efficiently sampleable.

\begin{lem}\label{lem:pldconv}
Let $\calL$ be a $M$-smooth function for some finite $M$. Then if $\beta_t = \beta$ for all $t$, then the stationary distribution of \eqref{eq:PLD} has pdf proportional to $\exp(-\beta \calL(\theta; \dataset)) \cdot \mathds{1}(\theta \in \calC)$.
\end{lem}

We recall that one can ensure smoothness by convolving $\calL$ (appropriately extended to all of $\mathbb{R}^p$) with the Gaussian kernel of finite variance~\citep[Appendix C]{FMTT18}. In particular, since we only need $M$ to be finite, we can take the convolution with the Gaussian kernel $\calN(\boldzero, \lambda^2 \mathbb{I}_p)$ for arbitrarily small $\lambda > 0$, and in turn the result of the convolution is $L/\lambda$-smooth (which is perhaps arbitrarily large but still finite) and differs from $\calL$ by an arbitrarily small amount everywhere in $\calC$.

We use the result by \cite{steinke2017tight} for our lower bound proof. We use their equivalent result for empirical mean (see equation (2) in \cite{steinke2017tight}) and for privacy parameters $(\epsilon,\delta)$ using a standard reduction~\citep{bun2018fingerprinting, steinke2015between}\footnote{\cite{steinke2017tight} present their result in the terms of population mean and  privacy parameters $(1,{\frac{1}{ns}})$.}:
\begin{thm}
\label{thm:steinke-ullman2017}
Fix $n,s,k \in \mathbb N$. Set $\beta = 1 + \frac{1}{2} \log \paren{\frac{s}{8\max\set{2k,28}}}$. Let $P^1,\cdots, P^s \sim \mathsf{Beta}(\beta,\beta)$ and let $X:= \set{\bfx_1, \cdots, \bfx_n}$  be such that $\bfx_i \in \set{0,1}^s$ for all $i \in [n]$, $\bfx_{i,j}$ is independent (conditioned on $P$) and $\E[\bfx_{i,j}] = P^j$ for all $i \in [n]$ and $j \in [s]$. Let $\mathcal M:(\set{0,1}^s)^n \to \set{0,1}^d$ be $(1, \frac{1}{ns})$-differentially private. Suppose $\norm{\calM(x)}_1 = \norm{\calM(x)}_2^2 =k$ for all $X$ with probability $1$ and 
\begin{align}
\E_{\mathcal M} \sparen{ \frac{1}{n}\sum_{i=1}^n \sum_{j \in [s] \atop \calM(x)^j = 1}  \bfx_{i,j}} \geq  { \frac{1}{n} \max_{S \subset [d] \atop |S|=k} \sum_{i=1}^n \sum_{u \in S}  \bfx_{i,u} } -  {\frac{k}{20} }.
    \label{eq:steinke-ullman-k}
\end{align}
Then $n \in \Omega \paren{  \sqrt{k} \log\paren{\frac{s}{k}}}$.
\end{thm}

\mypar{Results from statistics and machine learning}
We will sometimes use Fatou's lemma in our proofs. The form we will use is stated here for convenience:
\begin{lem}[Fatou's Lemma]
Let $\{X_i\}$ be a sequence of random variables such that there is some constant $c$ such that for all $i$, $Pr[X_i \geq c] = 1$. Then:

\[\Exp{\liminf_{i \rightarrow \infty} X_i} \leq \liminf_{i \rightarrow \infty} \Exp{X_i}.\]
\label{lem:Fatou}
\end{lem}

For our SCO bounds, we will use uniform stability. Uniform stability of a learning algorithm is a notion of algorithmic stability introduced to derive high-probability bounds on the generalization error. Formally, it is defined as follows:

\begin{defn}
[Uniform stability~\citep{bousquet2002stability}]
A mechanism $\calM$ is {\em $\mu(n)$-uniformly stable} with respect to $\ell$ if for any pair of databases $\dataset, \dataset'$ of size $n$ differing in at most one individual:

\[\sup_{d \in \tau} \left[\mathbb{E}_{\calM} \left[ \ell(\calM(\dataset), d)\right] - \mathbb{E}_{\calM} \left[ \ell(\calM(\dataset'), d)\right]\right] \leq \mu(n).\]
\label{def:unifStab}
\end{defn}

In this paper, we will need the following result.
\begin{lem}[\cite{bousquet2002stability}]\label{lem:emptopop-us}
Suppose $\calM$ is $\mu(n)$-uniformly stable. Then:

\[\mathbb{E}_{\dataset \sim \calD^n, \calM}[\riskpop(\calM(\dataset))] \leq \mathbb{E}_{\dataset \sim \calD^n, \calM}[\riskerm(\calM(\dataset))] + \mu(n).\]
\end{lem}

We also use KL divergence and TVD, and the relation between the two. 
\begin{defn}
The KL divergence (equivalent to the 1-R\'enyi divergence) between two distributions $P, Q$ is given by $R_1(P, Q) := \mathbb{E}_{x \sim P} \log \left(\frac{P(x)}{Q(x)}\right)$.
\end{defn}

\begin{lem}[Pinsker's inequality]
\label{lem:pinsker}
Let $TVD(P, Q)$ be the total variation distance between $P$ and $Q$. Then 

\begin{align*}TVD(P, Q) \leq \sqrt{\frac{1}{2} R_1(P, Q)}.\end{align*}
\end{lem}

Next, we define the LSI constant of a distribution.

\begin{defn}
A distribution $P$ satisfies LSI with constant $c$ if for all smooth functions $g: \mathbb{R}^p \rightarrow \mathbb{R}$ with $\mathbb{E}_{x \sim P}[g(x)^2]  < \infty$:
\begin{align*}\mathbb{E}_{x \sim P}[g(x)^2 \log (g(x)^2)] - \mathbb{E}_{x \sim P}[g(x)^2] \cdot\mathbb{E}_{x \sim P}[\log (g(x)^2)] \leq \frac{2}{c} \mathbb{E}_{x \sim P}[\ltwo{\nabla g(x)}^2].\end{align*}
\label{def:LSI}
\end{defn}
\section{Proofs from Section~\ref{sec:rashomonSampling} and Section~\ref{sec:sgld}}
\label{sec:deferredsectionrashomonsampling}

\subsection{Proof of Lemma~\ref{lem:renyifinite}}\label{app:langevin-deferred}
\begin{proof}
For ease of presentation, we will show a divergence bound between $\Theta_t$, $\Theta_t'$ which are the distributions of $\theta_t, \theta_t'$, and then describe how to modify the proof to show the same bound between $\Theta_{[0, t]}, \Theta_{[0, t]}'$.

Let $\Psi_{\dataset, m, i}$ be a map from (distributions over) $\mathbb{R}^p$ to (distributions over) $\mathbb{R}^p$ that takes the point $\theta$ to the distribution $\Pi_{\calC}\left(N\left(\theta - \left(\frac{\beta t}{m} dt \right) \grad \calL(\theta; \dataset), 2 \frac{t}{m} \mathbb{I} \right)\right)$, where $\Pi_{\calC}$ is the $\ell_2$-projection into $\calC$. 
It is well known (see e.g. \Cref{fact:gaussiandivergence}) that:

\[R_\alpha(N(0, \sigma^2 \mathbb{I}), N(\bfx, \sigma^2 \mathbb{I})) \leq \frac{\alpha \ltwo{\bfx}^2}{2\sigma^2}.\]

So by post-processing (Fact~\ref{fact:postprocessing}) and the Lipschitzness assumption, $R_\alpha(\Psi_{\dataset, m, i}(\theta), \Psi_{\dataset', m, i}(\theta)) $ is bounded by

\begin{align*}
    & R_\alpha \left(N \left(\theta - \left(\frac{\beta t}{m} \right) \grad \calL(\theta; \dataset), \frac{2t}{m} \mathbb{I}\right), N\left(\theta - \left(\frac{\beta t}{m}\right) \grad \calL(\theta; \dataset'), \frac{2t}{m} \mathbb{I} \right) \right) \\
    &=R_\alpha\left(N \left(\textbf{0}, \frac{2t}{m} \mathbb{I} \right), N \left( \left(\frac{\beta t}{m}\right)(\grad \calL(\theta; \dataset) - \grad \calL(\theta; \dataset')), 2\frac{t}{m} \mathbb{I} \right) \right) \\
    &\leq \frac{\alpha \Delta^2 }{4 } \cdot \frac{\left(\frac{\beta t}{m}\right)^2}{t/m}.
\end{align*}

Let $\Psi_{\dataset, m}$ denote the composition $\Psi_{\dataset, m, m} \circ \Psi_{\dataset, m, m-1} \circ \ldots \circ \Psi_{\dataset, m, 1}$. By \Cref{fact:composition}, we have

\[ R_\alpha(\Psi_{\dataset, m}(\Theta_0), \Psi_{\dataset', m}(\Theta_0)) \leq \sum_{i=1}^m \max_\theta \set{R_\alpha(\Psi_{\dataset, m, i}(\theta), \Psi_{\dataset', m, i}(\theta))}.\]

Plugging in the bound on $R_\alpha(\Psi_{\dataset, m, i}(\theta), \Psi_{\dataset', m, i}(\theta))$, we get

\[ R_\alpha(\Psi_{\dataset, m}(\Theta_0), \Psi_{\dataset', m}(\Theta_0)) \leq \frac{\alpha \Delta^2 }{4 } \cdot \frac{m}{t} \sum_{i=1}^m {\left(\frac{\beta t}{m}\right)^2} = \frac{\alpha  \beta^2 \Delta^2 t}{4}\]

Note that $\Theta_t = \lim_{m \rightarrow \infty} \Psi_{\dataset, m}(\Theta_0)$, and $\Theta'_t = \lim_{m \rightarrow \infty} \Psi_{\dataset', m}(\Theta_0)$. Since $\exp((\alpha - 1) R_\alpha(\calP, \calQ))$ is a monotone function of $ R_\alpha(\calP, \calQ)$ and is the expectation of a positive random variable, by Fatou's lemma we have:
\begin{align*}
    R_\alpha(\Theta_t, \Theta_t') &\leq \lim_{m \rightarrow \infty} R_\alpha(\Psi_{\dataset, m}(\Theta_0), \Psi_{\dataset', m}(\Theta_0)) \\
    & \leq  \frac{\alpha  \beta^2 \Delta^2 t}{4}.
\end{align*} 

This gives the bound on $R_{\alpha}(\Theta_t, \Theta_t')$. To obtain the same bound for $R_{\alpha}(\Theta_{[0, t]}, \Theta_{[0, t]}')$, we  modify $\Psi_{\dataset, m, i}$ so that instead of receiving  $\Theta_{(i-1)t/m}$ and outputting $\Theta_{it/m}$, it receives the joint distribution $\{\Theta_{jt/m}\}_{0 \leq j \leq i-1}$ and outputs $\{\Theta_{jt/m}\}_{0 \leq j \leq i}$ by appending the (also jointly distributed) variable 

\[\Theta_{it/m} =\Pi_{\calC}\left(N\left(\theta - \left(\frac{\beta t}{m} \right) \grad \calL(\Theta_{(i-1)t/m}; \dataset), 2 \frac{t}{m} \mathbb{I} \right)\right.\]

That is, we update $\Psi_{\dataset, m ,i}$ so it outputs the distributions of all iterates seen so far instead of just the distribution of the last iterate; the limiting value of the joint distribution $\{\Theta_{jt/m}\}_{0 \leq j \leq i}$ is then $\Theta_{[0, t]}$ according to $\cref{eq:PLD}$, and the same divergence bound holds.
\end{proof}

\subsection{Proof of Theorem~\ref{thm:expMech-rashomon-approx}}\label{sec:rashomon-privacy-proof}

We first need the following results, which let us analyze the LSI constant of the Gibbs distribution of interest easily. These results were originally stated for unconstrained domains defined over the space of real numbers, a proof for which can also be found in \cite{ledoux2001logarithmic} using the theory of semigroup (see Corollary 1.4, 1.6 and Lemma 1.2). For the general convex set,  $\calC$, we refer the readers to Theorem 2.1 in \cite{KM16}.

\jnote{I could not find the result in \cite{ledoux1999concentration}, but found in his 2001 paper that I cited along with the references}
\begin{lem}[Proposition 3 and Corollaire 2 in \cite{Bakry1985}]\label{lem:bakry-emery}
Let $P$ be the distribution with density proportional to $\exp(-f(x)) \cdot \mathds{1}(x \in \calC)$ for convex $\calC$. If $f$ is $m$-strongly convex, then $P$ satisfies LSI with constant $m$.
\end{lem}

\begin{lem}[Page 1184 in \cite{Holley1987LogarithmicSI}]\label{lem:holley-stroock}
Let $f, f'$ be two functions such that $$\sup_{\theta \in \calC} | f(\theta) - f'(\theta)| \leq \Delta.$$ Suppose the distribution with density proportional to $\mathds{1}(\theta \in \calC) \cdot \exp(-f)$ satisfies LSI with constant $c$. Then the distribution with density proportional to $\mathds{1}(\theta \in \calC) \cdot \exp(-f')$ satisfies LSI with constant $c \cdot \exp(-\Delta)$.
\end{lem}

The following result shows convergence of \eqref{eq:PLD} under LSI. 

\begin{lem}[Theorem 4 of \cite{vempala2019renyi}]\label{lem:vw19-convergence-kl}
Suppose $Q$ satisfies LSI with constant $c$. Then let $P_0$ be any initial distribution over $\theta_0$, and $P_t$ be the distribution over $\theta_t$ given by running \eqref{eq:LD} on $-\log q$, where $q$ is the density of $Q$, for $\beta = 1$. Then
$R_1(P_t, Q) \leq \exp(-2ct) \cdot R_1(P_0, Q).$
\end{lem}

Using these results, we first prove an LSI constant for the Gibbs distribution.

\begin{lem}\label{lem:rashomon-lsi}
Suppose we have non-negative $\ell(\theta; d)$ such that $\ell$ is $m$-strongly convex wrt $\theta$ for all $d$, and let $Q$ be the distribution with density proportional to 
$$\mathds{1}(\theta \in \calC) \cdot \exp\left(-\beta \max\left\{\threshold + \min\limits_{\theta^* \in \calC}\calL(\theta^*; D), \calL(\theta; D) \right\} \right).$$ Then $Q$ satisfies LSI with constant $c$, where

\begin{align*}
c := \begin{cases}
        \frac{m}{e \threshold} & \text{for } \beta \threshold > 1\\
        \beta m \exp(-\beta \threshold) & \text{for } \beta \threshold \leq 1
\end{cases}.
\end{align*}
\end{lem}
\begin{proof}
For $a \in [0, 1]$, consider the distribution with density proportional to $e^{-f(\theta)} \cdot \mathds{1}(\theta \in \calC)$ for 
\begin{align}
f(\theta) := \beta \paren{ a \cdot  \calL(\theta; D) + (1-a)\cdot \max\left\{\threshold + \min\limits_{\theta \in \calC}\calL(\theta; D), \calL(\theta; D)\right\} }.
\label{eq:functionf}
\end{align}

Since $\calL$ is $m$-strongly convex, the function  $\max\{\threshold + \min_{\theta \in \calC}\calL(\theta; D) , \calL(\theta; D)\}$ is a  convex function. Therefore, $f(\theta)$ defined in \cref{eq:functionf} is an $(a \beta m)$-strongly convex function. In other words, this distribution satisfies LSI with constant $a \beta m$ using \Cref{lem:bakry-emery}.

Now the distribution $Q$ has density proportional to $\exp(-f)$, where $$f' := \max\{\threshold + \min_{\theta \in \calC}\calL(\theta; D), \calL(\theta; D)\}.$$ 
When $\calL(\theta; D) > \threshold$, $f = f'$. Otherwise, since $\ell$ is non-negative, we have that $f$ differs from $f'$ by at most $a \beta \threshold$ everywhere. Then by \cref{lem:holley-stroock}, $Q$ satisfies LSI with constant $a \beta m \exp(-a \beta \threshold)$. 

Now, we maximize this bound over $a \in [0, 1]$. If $\beta \threshold \leq 1$, then this bound is maximized at $a = 1$ and we get an LSI constant of $\beta m \exp(-\beta \threshold)$. Otherwise, this bound is maximized at $a = 1 / \beta \threshold$ and we get an LSI constant of $\frac{m}{e \threshold}$. This completes the proof of \Cref{lem:rashomon-lsi}.
\end{proof}

Given the LSI constant of Rashomon sampler, we can use Lemma~\ref{lem:vw19-convergence-kl} to upper bound the time we need to run DP-LD.

\begin{lem}\label{lem:rashomon-convergence}
Under the assumptions of Theorem~\ref{thm:expMech-rashomon-approx}, let 
$$\calL'(\theta; D) := \max \left\{\threshold + \min\limits_{\theta^* \in \calC}\calL(\theta^*; D), \calL(\theta; D) \right\}.$$ 
Let $Q$ be the distribution with density proportional to $\exp(-\beta \calL')$. Let $P_0$ be the distribution over $\theta_0$ that is uniform over $\calC$. Let $P_t$ be the resulting distribution over $\theta_t$ given by running \eqref{eq:PLD} on $\calL'$. Then assuming $\threshold\leq L\ltwo{\calC}$:

\begin{itemize}
    \item If $\beta \threshold > 1$, for $t = \frac{e \threshold}{2 m} \log(\frac{2\beta L \ltwo{\calC} }{\delta^2})$, we have $TVD(P_t, Q) \leq \delta/2.$
    \item If $\beta \threshold \leq 1$, for $t = \frac{e}{2 \beta m} \log(\frac{2\beta L \ltwo{\calC} }{\delta^2})$, we have $TVD(P_t, Q) \leq \delta/2.$
\end{itemize}

\end{lem}
\begin{proof}
We wish to apply Lemma~\ref{lem:vw19-convergence-kl}, which was originally stated in the unconstrained case for \eqref{eq:LD}. One can see that it applies to running \eqref{eq:PLD} on $\calL'$ by the following argument: Consider extending $\calL'$ to $\mathbb{R}^p$ by defining $\calL'(\theta') = \calL'(\Pi_\calC(\theta')) + c \ltwo{\theta' - \calL'(\Pi_\calC(\theta'))}$ for $\theta' \not \in \calC$. As $c$ goes to infinity, the Gibbs distribution induced by the extended loss in the unconstrained setting approaches the Gibbs distribution induced by the loss in the constrained setting, and \eqref{eq:LD} approaches \eqref{eq:PLD}. 

We have:
\begin{align}
\begin{split}
R_1(P_0, Q) &\leq \max_{\theta \in supp(P_0)} \log(P_0(\theta) / Q(\theta))  \\
    & \leq \log\left(\frac{1}{\int\limits_{\theta \in \calC} 1\ d \theta} \cdot \frac{\int\limits_{\theta \in \calC} \exp(- \beta \max\{\calL(\theta^*; D) + \threshold, \calL(\theta; D)\})\ d \theta}{\min_{\theta \in \calC} \exp(-\beta \max\{\calL(\theta^*; D) + \threshold, \calL(\theta; D)\})}\right)\\
    & \leq \log\left(\frac{\exp(-\beta(\calL(\theta^*; D)  + \threshold)}{\exp(- \beta (L \ltwo{\calC} + \threshold))}\right) \\
    &\leq 2 \beta L \ltwo{\calC}.    
\end{split}
\label{eq:R_1P_0bound}
\end{align}

Here, we use the fact that each $\ell$ has minimum 0. 
Recall that 
$$\calL'(\theta; D) := \max \left\{\threshold + \min\limits_{\theta^* \in \calC}\calL(\theta^*; D), \calL(\theta; D) \right\}.$$ 

Now running \eqref{eq:PLD} on $\beta \calL'$, replacing $\beta$ in \eqref{eq:PLD} with $1$, is equivalent to running \eqref{eq:PLD} on $\calL'$. \jnote{This sentence needs to be elaborated.}
Therefore, applying Lemma~\ref{lem:vw19-convergence-kl} and Lemma~\ref{lem:rashomon-lsi}:

\begin{itemize}
    \item If $\beta \threshold > 1$, for $t = \frac{e \threshold}{2 m} \log\left(\frac{4 \beta L \ltwo{\calC} }{\delta^2}\right)$, we have:
    
    \begin{align*}R_1(P_t, Q) \leq R_1(P_0,Q) \cdot \exp\left(-\frac{2 m}{e \threshold} t \right) \leq 2 \beta L \ltwo{\calC}  \cdot \exp\left(-\frac{2 m}{e \threshold} t \right) = \delta^2/2,\end{align*}
    where the first inequality is using Lemma~\ref{lem:vw19-convergence-kl} and Lemma~\ref{lem:rashomon-lsi} and the second inequality is due to \cref{eq:R_1P_0bound}. The first bullet now follows using Pinsker's inequality (Lemma~\ref{lem:pinsker}).
    
    \item If $\beta \threshold \leq 1$, note that $\beta m \exp(-\beta \threshold) \geq \beta m / e$. So for $t = \frac{e}{2 \beta m} \log(\frac{4 \beta L \ltwo{\calC} }{\delta^2})$, we have:
    
    \begin{align*}R_1(P_t, Q) \leq R_1(P_0,Q) \cdot \exp\left(-\frac{2 m}{e \threshold} t \right) \leq  2 \beta L \ltwo{\calC}  \cdot \exp(-\frac{2 \beta m}{e} t) = \delta^2/2,\end{align*}
    where the first inequality is using Lemma~\ref{lem:vw19-convergence-kl} and Lemma~\ref{lem:rashomon-lsi} and the second inequality is due to \cref{eq:R_1P_0bound}. The second bullet now follows using Pinsker's inequality (Lemma~\ref{lem:pinsker}).
\end{itemize}

This completes the proof of \Cref{lem:rashomon-convergence}.
\end{proof}

In Lemma~\ref{lem:rashomon-convergence}, we used a uniform distribution on $\calC$ as our initialization (i.e., $P_0$ was uniform distribution over $\calC$). In the case where $\calC$ is \textit{unconstrained} and the losses satisfy $\ltwo{\nabla \ell(\theta; d) - \nabla \ell(\theta; d')} \leq L$ instead of Lipschiztness within $\calC$, we can obtain the same bound by letting $\calC$ be the convex hull of the minimizers of all per-example loss functions $\ell(\theta; d)$ and using a uniform distribution on $\calC$ as our initial distribution.

We now complete the proof of Theorem~\ref{thm:expMech-rashomon-approx}.
\begin{proof}
Let $\bfv_D(\theta)=\partial_{\theta}\max\{\calL(\theta;D), \calL(\theta^*; D)+\threshold\}$ be the sub-differential of the score function used in the Gibbs distribution.
Using the Lipschitz property and the smoothness assumption on $\ell(\theta;d)$,  for any two neighboring data data sets $D$ and $D'$, we have 

\begin{equation}\ltwo{\bfv_D(\theta)-\bfv_D(\theta)}\leq 3 \cdot \max\left\{\frac{L}{n},\sqrt{\frac{M\threshold}{2}}\right\}.
\label{eq:sends}
\end{equation}
 In particular, the Lipschitzness bounds the sensitivity between the gradients of $\calL(\theta;D)$ and $\calL(\theta;D')$ by $L/n$. Using smoothness, for $D$, the sensitivity between the gradients of $\calL(\theta;D)$ and\\ $\max \left\{\threshold + \min\limits_{\theta^* \in \calC}\calL(\theta^*; D), \calL(\theta; D) \right\}$ is bounded by $\sqrt{M\threshold/2}$. Similarly, we have the sensitivity between the gradients of $\calL(\theta;D')$ and $\max \left\{\threshold + \min\limits_{\theta^* \in \calC}\calL(\theta^*; D'), \calL(\theta; D') \right\}$ bounded by $\sqrt{M\psi/2}$.

By Lemma~\ref{lem:renyifinite}, \Cref{fact:renyitoapx}, and sensitivity bound in~\cref{eq:sends}, in order for $\{\Theta_{t'}\}_{t' \in [0, t]}$ to satisfy $(\epsilon, \delta)$-DP, it suffices if 
\begin{equation}
\beta \leq \frac{\epsilon}{6 \sqrt{2} \cdot \max\left\{\frac{L}{n},\sqrt{\frac{M\threshold}{2}}\right\} \sqrt{t \log(1/\delta)}}.
\label{eq:9092}
\end{equation}

Suppose $\beta \threshold \leq 1$. Recall that, in Lemma~\ref{lem:rashomon-convergence},   $t = \frac{e}{2 \beta m} \log(\frac{4 \beta L \ltwo{\calC} }{\delta^2})$.   Plugging in this value of $t$ in~\cref{eq:9092} and observing that $\threshold\leq\frac{L\ltwo{\calC}}{2}$, it suffices to ensure $(\epsilon,\delta)$-DP if 
\begin{equation}
    \beta=\tilde{O}\left(\frac{\epsilon^2 n^2 m }{\max\left\{L,n\sqrt{{M\threshold}/2}\right\}^2 \log(L\ltwo{\calC}/\delta^2)\log(1/\delta)} \right).
    \label{eq:9090}
\end{equation}

We now consider two cases (on top of the constraint that $\psi \leq L\norm{\calC}_2/2$): 
\begin{enumerate}
    \item [(i)] $\threshold>\frac{2L^2}{Mn^2}$, and 
    \item [(ii)] $\threshold\leq\frac{2L^2}{Mn^2}$.
\end{enumerate} 

In the case (i), in order to satisfy $\beta\threshold\leq 1$ it is sufficient to set $\beta = \tilde{\Theta}\left(\frac{\epsilon^2 (m/M)}{ \log((L\ltwo{\calC}-\threshold)/\delta^2)\log(1/\delta)}\cdot \frac{1}{\threshold} \right).$ In the second, case it sufficient to set $\beta = \tilde{\Theta}\left(\frac{\epsilon^2 n^2 m }{L^2 \log((L\ltwo{\calC}-\threshold)/\delta^2)\log(1/\delta)} \right)$. 

We will not analyze the setting when $\beta\threshold>1$. From Lemma~\ref{lem:rashomon-convergence}, it is not hard to observe that it will not provide any better conditions on $\threshold$ and $\beta$ to ensure $(\epsilon,\delta)$-DP.

So, we have that for the choices of $\threshold, \beta, t$ in Theorem~\ref{thm:expMech-rashomon-approx}, outputting $\{\Theta_t'\}_{t' \in [0, t]}$ satisfies $(\epsilon, \delta)$-DP. Furthermore, by Lemma~\ref{lem:rashomon-convergence}, this gives the desired bound on total variation distance completing the proof of Theorem~\ref{thm:expMech-rashomon-approx}.
\end{proof}

\subsection{Probability of Hitting the Rashomon Set}
\label{sec:probRashomon}
\begin{thm}
Let $\privT$ be the model output by a Rashomon sampler. Let  $\Pr[\calL(\privT;D)\leq\min\limits_{\theta\in\calC}\calL(\theta;D)+\threshold+\gamma]\geq 1 - \phi$. Then for Rashomon set $\rashomonemp$,
$
\Pr\left[\privT\in\rashomonemp\right]\geq 1- \phi - \frac{p \gamma}{\threshold}.
$
\label{thm:hitting}
\end{thm}
\begin{proof}
Consider the differential conic region $\Omega$  centered at the true minimizer $\theta^*=\argmin\limits_{\theta\in\calC}\calL(\theta;D)$. Let $A$ be the cone from this region with height $r$, and $B$ be the one with height $(r+\Delta)$, with $r$ and $\Delta$ chosen such that $\theta$ in the boundaries of the cones $A$ and $B$ satisfy $\calL(\theta;
D)=\threshold + \min_\theta \calL(\theta; D)$ and $\calL(\theta;D)=\threshold+\gamma+ \min_\theta \calL(\theta; D)$ respectively. By law of total probability, it suffices to show that conditioned on sampling $\theta$ in any such cone, the desired probability bound holds. By the maximality condition on the Rashomon sampler, 

\begin{align*} \Pr\left[\privT\in\rashomonemp | \theta \in \Omega\right]\geq 1 - \phi - \frac{{\sf Vol}(A)}{{\sf Vol}(B)} = 1 - \phi - \left(1 - \frac{\Delta}{r + \Delta}\right)^p \geq 1 - \phi - \frac{p\Delta}{r}.\end{align*}

Note that within $\Omega$, we have $\calL(\theta; D) = f(\ltwo{\theta - \theta^*})$ where $f$ is a convex function. By convexity, we have:

\begin{align*}f(r) - f(0) \leq f'(r) \cdot r \rightarrow f'(r) \geq \frac{f(r) - f(0)}{r} = \threshold / r. \end{align*}
\begin{align*}\gamma = f(r + \Delta) - f(r) \geq f'(r) \cdot \Delta \geq \frac{\threshold \Delta}{r} \implies \Delta / r \leq \gamma / \threshold.\end{align*}

This completes the proof of Theorem~\ref{thm:hitting}.
\end{proof}

\subsection{Proof of Theorem~\ref{thm:sgld-sampler-tvd}}\label{app:kernel}

We start by showing some properties of the convolved loss function:

\begin{thm}\label{thm:kernel}
Let $\calL$ be a convex and $M$-smooth loss, let $\calL'(\theta) = \max\{\calL(\theta), \calL(\theta^*) + \threshold\}$ and let $\widetilde{\calL}'(\theta) := \mathbb{E}_{\xi \sim N(0, \lambda^2 \mathbb{I}_p)}\left[\calL'(\theta + \xi)\right]$. Then:
\begin{enumerate}
    \item $\max_{\theta} \left|\widetilde{\calL}'(\theta) - \calL'(\theta)\right| \leq 2\threshold + Mp\lambda^2$. 
    \item $\widetilde{\calL}'(\theta)$ is $M + \frac{\sqrt{\threshold M / 2}}{ \lambda}$-smooth.
\end{enumerate}
\end{thm}
\begin{proof}
Let $\widetilde{\calL}(\theta) := \mathbb{E}_{\xi \sim N(0, \lambda^2 \mathbb{I}_p)}\left[\calL(\theta + \xi)\right]$. Then:

\[
    \max_{\theta} \left|\widetilde{\calL}'(\theta) - \calL'(\theta)\right| \leq 2 \threshold + \max_{\theta} \left|\widetilde{\calL}(\theta) - \calL(\theta)\right|.
\]
We bound the second term for all $\theta$ simultaneously as follows:
\begin{align*}
 \left|\widetilde{\calL}(\theta) - \calL(\theta)\right| &= \left|\mathbb{E}_{\xi \sim N(0,\lambda^2 \mathbb{I}_p)}\left[\calL(\theta + \xi) - \calL(\theta)\right]\right| \\
 &\stackrel{(*)}{=} \left|\mathbb{E}_{\xi \sim N(0,\lambda^2 \mathbb{I}_p)}\left[\calL(\theta + \xi) - \calL(\theta) - \langle\nabla \calL(\theta), \xi\rangle\right]\right| \\
 &\stackrel{(**)}{\leq} \mathbb{E}_{\xi \sim N(0,\lambda^2 \mathbb{I}_p)}\left[M \ltwo{\xi}^2\right] = M p \lambda^2.
\end{align*}

In equality $(*)$, we use the fact that $\xi$ is mean 0, and in inequality $(**)$, we use the fact that $0 \leq \calL(\theta + \xi) - \calL(\theta) - \langle\nabla \calL(\theta), \xi\rangle \leq M \ltwo{\xi}^2$ by convexity and $M$-smoothness of $\calL$. This gives the first part of the theorem. 

For the second part, fix any $\theta_1, \theta_2$. 
The smoothness parameter of $\widetilde{\calL}'$ is at most the smoothness parameter of $\widetilde{\calL}$ (which is at most $M$, since $\calL$ is $M$-smooth), plus the smoothness of $\widetilde{\calL}' - \widetilde{\calL}$. 

Now $\widetilde{\calL}' - \widetilde{\calL}$ is the convolution of the Gaussian kernel with the function $\calL' - \calL$, which is $0$ outside of Rashomon set, $\calG$ and equal to $\calL(\theta^*) + \threshold - \calL(\theta)$ inside the Rashomon set, $\calG$. By $M$-smoothness of $\calL$, $\calL' - \calL$ is globally $\sqrt{\frac{\threshold M}2}$-Lipschitz. Then, by Theorem 33 of \cite{FMTT18}, $\widetilde{\calL}' - \widetilde{\calL}$ is $\frac{1}{\lambda}\sqrt{\frac{\threshold M}2}$-smooth, giving the second part of the theorem.

This completes the proof of Theorem~\ref{thm:kernel}.
\end{proof}

Similarly to Lemma~\ref{lem:rashomon-lsi} we have the following:

\begin{lem}\label{lem:kernel-lsi}
Let $\widetilde{\calL}'$ be defined as in Theorem~\ref{thm:kernel}, for some $\lambda \geq 0$. Then the distribution with density proportional to $\exp(-\beta \cdot \widetilde{\calL}')$ satisfies LSI with constant $c$ for
\[ c := \beta m \exp(-\beta (3 \threshold + M p \lambda^2)).
\]
\end{lem}

Recall the SGLD equation:
\begin{equation}\label{eq:dynamics}
    \theta_{t+1} = \theta_t - \eta \beta \nabla \calL(\theta_t; D) + \xi_t, \quad \xi_t \sim \calN(0, 2 \eta I_p).
\end{equation}

We can now apply the following result of \cite{ChewiLSILD}:

\begin{thm}
[Theorem 4 of \cite{ChewiLSILD}]
\label{lem:sgld-renyi}
For any $T$ and $\eta$, let $\Theta_T$ be the distribution of $\theta_{T\eta}$ given by the solution to \eqref{eq:LD}, and let $\Theta_T'$ be the distribution of $\theta_T$ given by \eqref{eq:dynamics}. Suppose $\calL$ is $M$-smooth and satisfies LSI (Definition~\ref{def:LSI}) with constant $c$. Then for any $\gap > 0, \alpha \geq 2$,

\[\eta = O\left(\frac{c \gap}{p \alpha M^2 } \cdot \min\left\{\frac{1}{\log(\alpha)}, \frac{p}{\alpha \gap}\right\}\right), \quad \text{and} \quad T = \Omega\left(\frac{p \alpha^2 M^2}{c^2 \gap}r \cdot \max\left\{\log(\alpha), \frac{\alpha \gap}{p}\right\}\right),\]
we have $R_\alpha(\Theta'_T, \Theta_\infty) \leq \gap.$
\end{thm}

We can now prove a general result in terms of R\'enyi divergence:

\begin{thm}\label{thm:sgld-sampler-renyi} 
Suppose $\ltwo{\nabla \calL(\theta; D) - \nabla \calL(\theta; D')} \leq \Delta$, and that each individual loss function $\ell$ is $m$-strongly convex and $M$-smooth. Let $0 \leq \lambda \leq \sqrt{\frac{\threshold}{Mp}}$ and let $$\widetilde{\calL} := \mathbb{E}_{\xi \sim N(0, \lambda^2 \mathbb{I}_p)} \left[\min\{\calL(\theta + \xi; D), \min\limits_{\theta^* \in \calC}\calL(\theta^*; D) + \threshold\}\right].$$ 
Let $Q$ be the (unconstrained) Gibbs distribution of $\beta \cdot \widetilde{\calL}$. Let $\Theta_T'$ be the solution to \eqref{eq:dynamics} starting from the distribution $\Theta_0'$, run on the loss $\widetilde{\calL}$. Fix any $r > 0$.
Then for 
\begin{align*}
\beta &= O\left(\frac{\epsilon^2 m}{\max\{\Delta^2, M \threshold\} \alpha \log(R_\alpha(\Theta_0', Q) / \gap) \log(1/\delta)}\right), \\
T &= \Omega\left(\frac{p \alpha^4 \log^3(R_\alpha(\Theta_0', Q) / \gap) (M^2 + \frac{M \threshold}{\lambda^2}) \max\left\{\Delta^4 , M^2 \threshold^2\right\} \log^2(1/\delta)}{\epsilon^4 m^4 \gap} \cdot \max\left\{\log(\alpha), \frac{\alpha \gap}{p}\right\}\right)
\end{align*}
and an appropriate choice of $\eta$, outputting a sample from $\Theta_T'$ is $(\epsilon, \delta)$-DP. Furthermore,

\[R_\alpha(\Theta_T', Q) \leq \gap .\]
\end{thm}
\begin{proof}
In Theorem~\ref{lem:sgld-renyi} we need:

\[T\eta = \Theta\left(\frac{\alpha}{c}\log\left( \frac{R_\alpha(\Theta_0', Q)}{\gap} \right) \right)\]

Plugging in Lemmas~\ref{lem:sgld-privacy} and~\ref{lem:kernel-lsi} for outputting $\Theta_T'$ from \eqref{eq:dynamics} to be $(\epsilon, \delta)$-DP it suffices if:

\[\beta = O\left(\frac{\epsilon^2 m \exp(-\beta (3 \threshold + M p \lambda^2))}{\max\{\Delta^2, M \threshold\} \alpha \log(R_\alpha(\Theta_0', Q) / \gap) \log(1/\delta)}\right).\]

If we assume $\lambda \leq \sqrt{\frac{\threshold}{Mp}}$, for $\epsilon \leq 1$ this condition implies $\beta = O(\frac{1}{\threshold + M p \lambda^2})$. So this can be simplified to:
\[\beta = O\left(\frac{\epsilon^2 m}{\max\{\Delta^2, M \threshold\} \alpha \log(R_\alpha(\Theta_0', Q) / \gap) \log(1/\delta)}\right).\]

Now we can plug in this value of $\beta$, the smoothness bound from Theorem~\ref{thm:kernel}, Lemma~\ref{lem:kernel-lsi}, and our assumption on $\lambda$ into the number of iterations $T$ to get an iteration complexity requirement of:

\[T = \Omega\left(\frac{p \alpha^4 (M^2 + \frac{M \threshold}{\lambda^2}) \max\{\Delta^4 , M^2 \threshold^2\} \log^2(1/\delta)}{\epsilon^4 m^4 \gap}\log^3(R_\alpha(\Theta_0', Q) / \gap) \max\{\log(\alpha), \frac{\alpha \gap}{p}\}\right).\]
This completes the proof of Theorem~\ref{thm:sgld-sampler-renyi}. 
\end{proof}

In order to bound the initial divergence, we use $\Theta_0'$ that is a normal distribution centered at a point in the convex hull of minimizers of $\ell$, i.e. the convex hull of $\{\argmin_{\theta} \ell(\theta; d) : d \in \tau\}$. 

\begin{lem}\label{lem:initialdivergencebound}
Suppose $\ltwo{\nabla \calL(\theta; D) - \nabla \calL(\theta; D')} \leq \Delta$, and let $\theta_0$ be an arbitrary point in the convex hull of $\{\argmin_{\theta} \ell(\theta; d) | d \in \tau\}$. Let $P$ be the Gibbs distribution on $\beta \cdot \widetilde{\calL}'$ as defined in Theorem~\ref{thm:kernel}. Then for $\alpha \geq 2$:

\[R_\alpha\left(N\left(\theta_0, \frac{1}{\beta M} \mathbb{I}_p\right), P\right) = O\left(\frac{\alpha \Delta \beta^2 M^2}{m^2} + p \ln \paren{\frac{M}m} +  \beta \threshold\right).\]
\end{lem}
\begin{proof}
By e.g. Lemma 16 in \cite{GaneshT20}, if $\theta^*$ is the true minimizer of $\calL$ and $P_1$ is the Gibbs distribution on $\beta \cdot \calL$, for all $\alpha \geq 1$:

\[R_\alpha\left(N\left(\theta^*, \frac{1}{\beta M} \mathbb{I}_p\right), P_1\right) \leq \frac{p \ln \paren{\frac{M}m}}{2}.\]

By Theorem~\ref{thm:kernel}, $\beta \cdot \calL$ and $\beta \cdot \widetilde{\calL}'$ differ by at most $4 \beta \threshold$ everywhere if $\lambda \leq \sqrt{\frac{\threshold}{Mp}}$. Then:

\[R_\infty(P_1, P) \leq 8 \beta \threshold.\]

So by the approximate triangle inequality for R\'enyi divergences (Fact~\ref{fact:renyitriangle}):

\[R_\alpha\left(N(\theta^*, \frac{1}{\beta M} \mathbb{I}_p), P \right) \leq \frac{p \ln \paren{\frac{M}m}}{2} +  8 \beta \threshold\]

Finally, by the assumption on $\nabla \calL$, all points in $\{\argmin_{\theta} \ell(\theta; d) | d \in \tau\}$ are distance at most $\Delta / m$ apart by strong convexity. Then by Lemma~\ref{fact:gaussiandivergence}:

\[R_\alpha\left(N\left(\theta_0, \frac{1}{\beta M} \mathbb{I}_p\right), N\left(\theta^*, \frac{1}{\beta M} \mathbb{I}_p\right) \right) \leq \frac{\alpha \Delta \beta^2 M^2}{2 m^2} \]

Applying Fact~\ref{fact:renyitriangle} again gives Lemma~\ref{lem:initialdivergencebound}.
\end{proof}

Now by Pinsker's inequality and monotonicity of R\'enyi divergences, we can use Lemma~\ref{lem:initialdivergencebound}, $\alpha = 2$, and $\gap = \sqrt{\delta / 2}$ to get a total variation distance bound of $\delta$ to the stationary distribution in Theorem~\ref{thm:sgld-sampler-renyi}, proving Theorem~\ref{thm:sgld-sampler-tvd}.
\section{Lower Bound on DP-ERM for Non-Convex Losses}
\label{sec:lb}

In this section, we show the following lower bound on the excess empirical risk for $1$-Lipschitz non-convex loss functions. The lower bound implies that that there is no advantage, in terms of the dependence on dimensions ($p$), to move from $\epsilon$-DP to $(\epsilon,\delta)$-DP. 

\begin{theorem}
\label{thm:lower}
Let $\epsilon \leq 1$,  $2^{-\Omega(n)} \leq \delta \leq 1/n^{1+ \Omega(1)}$, and $B(\mathbf 0,1)$ be a unit Euclidean ball centered at origin. 
Then there exists $1$-Lipschitz non-convex function $\mathcal L: B(0,1) \times \mathcal X \rightarrow \mathbb R$ and a dataset \footnote{The dataset, $\dataset = \set{d_1, \cdots, d_n}$ is such that $d_i \in \set{0,1}^s$ for all $i \in [n]$, $d_{i,j}$ is independent (conditioned on $P$) and $\E[d_{i,j}] = P^j$ for all $i \in [n]$ and $j \in [s]$. Here $\calP$ is the distribution that is defined in \Cref{thm:steinke-ullman2017}.} $\dataset = \set{d_1, \cdots, d_n}$ such that for every $p \in \mathbb N$, there is no $(\epsilon,\delta)$-differentially private algorithm $\mathcal A$ that outputs $\thetapriv$ such that 
\begin{align}
\mathbb E \sparen{\mathcal L(\thetapriv; \dataset) - \min_{\theta \in \mathbb B_p(1)} \mathcal L(\theta; \dataset) } = o \paren{\frac{p\log \paren{{1}/{\delta}}}{n\epsilon} },
\label{eq:nonconvexloss}    
\end{align} 
\end{theorem}
\begin{proof}
We first perform two translations of  \Cref{thm:steinke-ullman2017}: first from $(1, \frac{1}{ns})$ to $(\epsilon,\delta)$ from~\cite{steinke2015between} and then from sample complexity to a result stated in the terms of accuracy bound. A direct corollary of \Cref{thm:steinke-ullman2017} with $k=1$ is as follows: for every $s \in \mathbb N$, no $(\epsilon,\delta)$-differentially private algorithm on input $X$ satisfying the premise of \Cref{thm:steinke-ullman2017} outputs an index $j \in [s]$ such that 
\begin{align}
\E_{\mathcal M} \sparen{ \frac{1}{n}  \sum_{i=1}^n \bfx_{i,j}} - { \max_{u \in [s]} \frac{1}{n}\sum_{i=1}^n \bfx_{i,u}}  = o \paren{\frac{1}{n\epsilon}\log(s) \log \paren{{1}/{\delta}}},
\label{eq:steinke_ullman_top}    
\end{align}
where $\epsilon \leq 1$ and $2^{-\Omega(n)} \leq \delta \leq 1/n^{1+\Omega(1)}$.

Using this lower bound on top-selection, we give our lower bound by defining an appropriate non-convex loss function. In particular, we define a packing over the $p$-dimensional Euclidean ball such that there is an bijective mapping between the centers of the packing and $[s]$. Then the function attains the minimum at the center of packing which corresponds to the coordinate $j \in [s]$ with maximum frequency. Since the size of the $\alpha$-net is $\approx 1/\alpha^p$ and there is a bijective mapping, this gives a lower bound using \cref{eq:steinke_ullman_top}.

Let $B(\mathbf 0,1)$ be the $p$-dimensional Euclidean ball centered at origin and let $\alpha \in (0,1/2)$ be a constant. Consider an $\alpha$-packing with centers $C = \set{\bfc_1, \bfc_2, \cdots,  }$. It is known that the size of such packing, $N(\alpha)$ is  $\paren{\frac{1}{\alpha}}^p \leq N(\alpha) \leq \paren{\frac{3}{\alpha}}^p$. Let $s=N(\alpha)$. Further, let $f: B(\mathbf 0,1) \to \set{1,\cdots, s}$ be an injective function defined as follows:
\[
f(\theta) = \set{j :  \mathbf{c}_j = \argmin_{\mathbf c \in C} \norm{\theta - \mathbf c}_2}.
\]
In particular, $f$ is the function that maps a point on the unit ball to its closest point in $C$.

We now define our loss function as follows:
\begin{align}
\mathcal{L} (\theta; \dataset) := \frac{1}{n} \sum_{d_i \in D} \ell(\theta;d_i) ~\text{where}~   \ell(\theta; d_i) = \min_{\mathbf c_j \in C} \paren{ \frac{\norm{\theta - \mathbf{c}_j}}{\alpha} - 1} {d_{i,j}} .
\label{eq:nonconvexlossfunction}
\end{align}

For Lipschitz property, note that each loss function is $1/\alpha$-Lipschitz because the gradient when it is defined is just $\frac{\theta - \bfc_j}{\alpha \ltwo{\theta - \bfc_j}}$. We prove it formally. 

Consider any $\theta, \theta'$ in $B(\boldzero, 1)$ and a data point $d_i \in D$. We wish to show $|\ell(\theta; d_i) - \ell(\theta'; d_i)| \leq \frac{1}{\alpha}\ltwo{\theta - \theta'}$. We can split the line segment from $\theta$ to $\theta'$ into a sequence of line segments $$(\theta_0, \theta_1), (\theta_1, \theta_2), \ldots, (\theta_{k-1}, \theta_k),$$ where $\theta_0 = \theta, \theta_k = \theta'$, such that for any line segment $(\theta_m, \theta_{m+1})$, $\theta_m$ and $\theta_{m+1}$ share a minimizer in $C$ of $\left(\frac{\ltwo{\theta - \bfc_j}}{\alpha}\right)d_{i, j}$.\footnote{In particular, for each $\bfc_j$ let $B_j$ be the set of points in $B(\boldzero, 1)$ such that $\bfc_j$ is a minimizer of $\left(\frac{\ltwo{\theta - \bfc_j}}{\alpha}\right)d_{i, j}$. We can split the line segment from $\theta$ to $\theta'$ at each point where it enters or leaves some $B_j$ to get this sequence of line segments, and by this construction each line segment's endpoints are both in $B_j$ for some $j$.}

It now suffices to show $|\ell(\theta_m; d_i) - \ell(\theta_{m+1}; d_i)| \leq \frac{1}{\alpha} \ltwo{\theta_m - \theta_{m+1}}$ for each $m$, since we then have:

\[|\ell(\theta; d_i) - \ell(\theta'; d_i)| \leq \sum_{m=0}^{k-1}  |\ell(\theta_{m}; d_i) - \ell(\theta_{m+1}; d_i)| \leq \frac{1}{\alpha}\sum_{m=0}^{k-1} \ltwo{\theta_{m} - \theta_{m+1}} = \frac{1}{\alpha}\ltwo{\theta - \theta'}.\]

Let $\bfc_j$ be a shared minimizer of $\left(\frac{\ltwo{\theta - \bfc_j}}{\alpha}\right)d_{i, j}$ for $\theta_m$ and $\theta_{m+1}$. If $d_{i, j} = 0$, then trivially $|\ell(\theta_m; d_i) - \ell(\theta_{m+1}; d_i)| \leq \frac{1}{\alpha} \ltwo{\theta_m - \theta_{m+1}}$. Otherwise $d_{i,j} = 1$ and by triangle inequality, we have:

\[|\ell(\theta_m; d_i) - \ell(\theta_{m+1}; d_i)| = \left|\frac{\ltwo{\theta_{m} - \bfc_j}}{\alpha} - \frac{\ltwo{\theta_{m+1} - \bfc_j}}{\alpha}\right| \leq \frac{1}{\alpha} \ltwo{\theta_m - \theta_{m+1}}.\]

Now let us suppose there is an $(\epsilon,\delta)$-differentially private algorithm $\mathcal A$ that on input a non-convex function $\mathcal L$ and $n$ data points $\set{d_1,\cdots, d_n}$, outputs a $\thetapriv$ such that 
\begin{align}
\E_{\mathcal A} \sparen{\mathcal L(\thetapriv; \dataset)} - {\min_{\theta \in  B(1)} \mathcal L(\theta; \dataset) }   = o \paren{\frac{p\log \paren{{1}/{\delta}}}{n\epsilon} },
\label{eq:nonconvexlossalternate}    
\end{align}
where $ D = \set{d_1, \cdots,d_n}$.

We will construct an algorithm that uses $\mathcal A$ as subroutine and solve top-selection problem with an error $o(\log (s))$, contracting the lower bound of \Cref{thm:steinke-ullman2017}.

Algorithm $\mathcal B$:
\begin{itemize}
    \item On input $X =\set{\bfx_1,\cdots, \bfx_n}$, invokes $\mathcal A$ on the function defined by \cref{eq:nonconvexlossfunction} and data points $X$ to get $\thetapriv$ as output.
    \item Output $f(\thetapriv)$.
\end{itemize}

Since the last step is post-processing, $\mathcal B$ is $(\epsilon,\delta)$-differentially private. We now show that if $\mathcal A$ outputs a $\thetapriv$ satisfying \cref{eq:nonconvexlossalternate}, then $j:=f(\thetapriv)$ satisfies \cref{eq:steinke_ullman_top} leading to a contradiction.

First note that, for any $\mathbf c \in C$ and all $\theta \in \mathbb B_p(\mathbf c,\alpha)$ such that $\norm{\theta - \mathbf c}_2 \leq \frac{\alpha}{2}$, 
\[
\mathcal L(\mathbf c; \dataset) = - \frac{1}{n} \sum_{i=1}^n \bfx_{i,f(\mathbf c)} \leq   \mathcal L(\theta; \dataset). 
\]

Therefore, 
\begin{align*}
\mathcal L(\theta^*;X) := \min_{\mathbf c \in C} {\mathcal L(\mathbf c,X)} = \min_{\mathbf c \in C} \paren{ -\frac{1}{n} \sum_{i=1}^n \bfx_{i,f(\mathbf c)} } \end{align*}

This implies that 
\[
f(\theta^*) = \argmax_{1 \leq j \leq s} \frac{1}{n} \sum_{i=1}^n \bfx_{i,j},  
\]
which is exactly the top-selection problem. Therefore, \cref{eq:nonconvexlossalternate} implies \cref{eq:steinke_ullman_top} 
because $p  \log \left(\frac{1}\alpha \right) \leq  \log(s) \leq p \log \left(\frac{3}\alpha \right)$ and $\alpha \in (0,1/2)$ is a constant.

\end{proof}

\end{document}